\documentclass[11pt]{article}
\usepackage[margin=1in]{geometry}

\usepackage{amssymb}
\usepackage{latexsym}
\usepackage{mathtools}

\usepackage{amsmath}
\usepackage{amsthm}
\usepackage{thmtools}
\usepackage{thm-restate}
\usepackage{xcolor}
\usepackage[round]{natbib}
\usepackage{nameref}
\usepackage{hyperref}
\hypersetup{colorlinks, linkcolor=blue, filecolor = blue, citecolor = blue, urlcolor = magenta}
\usepackage[noabbrev,capitalize]{cleveref}
\usepackage{graphicx}

\usepackage[extdef=true]{delimset}
\usepackage{enumitem}
\usepackage{algorithm}
\usepackage{algorithmic}
\usepackage{nicefrac}

\theoremstyle{plain}
\newtheorem{theorem}{Theorem}[section]

\newtheorem{lemma}[theorem]{Lemma}
\newtheorem{corollary}[theorem]{Corollary}
\theoremstyle{definition}
\newtheorem{definition}[theorem]{Definition}
\newtheorem{assumption}{Assumption}
\theoremstyle{remark}
\newtheorem{remark}[theorem]{Remark}

\usepackage[textsize=tiny]{todonotes}

\usepackage{bbm}
\usepackage{cleveref}

\renewcommand{\norm}[1]{\left\lVert#1\right\rVert_2}
\newcommand\inner[2]{\left\langle #1, #2 \right\rangle}

\newcommand{\ifcomments}{\iftrue}

\newcommand\bb[0]{\mathbb{B}}

\newcommand\ee[0]{\mathbb{E}}

\newcommand\rr[0]{\mathbb{R}}
\renewcommand\ss[0]{\mathbb{S}}

\newcommand\aaa[0]{\mathcal{A}}

\newcommand\ddd[0]{\mathcal{D}}
\newcommand\eee[0]{\mathcal{E}}
\newcommand\fff[0]{\mathcal{F}}
\renewcommand\ggg[0]{\mathcal{G}}
\newcommand\hhh[0]{\mathcal{H}}

\newcommand\ooo[0]{\mathcal{O}}
\newcommand\ppp[0]{\mathcal{P}}

\newcommand\rrr[0]{\mathcal{R}}

\newcommand\xxx[0]{\mathcal{X}}

\newcommand\rb[1]{\left(#1\right)}
\renewcommand\sb[1]{\left[#1\right]}
\newcommand\cb[1]{\left\{#1\right\}}

\allowdisplaybreaks

\title{Federated Online and Bandit Convex Optimization}

\author{
Kumar Kshitij Patel 
\thanks{TTIC; \texttt{kkpatel@ttic.edu}.}
\and
Lingxiao Wang
\thanks{TTIC; \texttt{lingxw@ttic.edu}.}
\and
Aadirupa Saha%
\thanks{Apple ML Research; \texttt{aadirupa@ttic.edu}. This work was done while Aadirupa Saha was visiting TTIC.}
\and 
Nati Sebro%
\thanks{TTIC; \texttt{nati@ttic.edu}.} 
}
\date{}

\usepackage{lipsum}

\newcommand\blfootnote[1]{%
  \begingroup
  \renewcommand\thefootnote{}\footnote{#1}%
  \addtocounter{footnote}{-1}%
  \endgroup
}

\usepackage{enumitem,kantlipsum}
\usepackage[algo2e,ruled]{algorithm2e}

\begin{document}

\maketitle

\begin{abstract}
We study the problems of \emph{distributed online and bandit convex optimization} against an adaptive adversary. We aim to minimize the average regret on $M$ machines working in parallel over $T$ rounds with $R$ intermittent communications. Assuming the underlying cost functions are convex and can be generated adaptively, our results show that \emph{collaboration is not beneficial when the machines have access to the first-order gradient information at the queried points}. This is in contrast to the case for stochastic functions, where each machine samples the cost functions from a fixed distribution. Furthermore, we delve into the more challenging setting of \emph{federated online optimization with bandit (zeroth-order) feedback}, where the machines can only access values of the cost functions at the queried points. The key finding here is \emph{identifying the high-dimensional regime where collaboration is beneficial and may even lead to a linear speedup in the number of machines}. 
We further illustrate our findings through federated adversarial linear bandits by developing novel distributed single and two-point feedback algorithms. 
Our work is the first attempt towards a systematic understanding of federated online optimization with limited feedback, and it attains tight regret bounds in the intermittent communication setting for both first and zeroth-order feedback. Our results thus bridge the gap between stochastic and adaptive settings in federated online optimization.
\end{abstract}

\blfootnote{A version of this work appeared in the Fortieth International Conference on Machine Learning, ICML 2023 proceedings. A preliminary version of this work was also presented at the Optimization for Machine Learning Workshop (OPT-2022) at the Thirty-sixth Conference on Neural Information Processing Systems, NeurIPS 2022.}

\section{Introduction}\label{sec:introduction}
We consider the following distributed regret minimization problem on $M$ machines over a horizon of length $T$: 
\begin{align}\label{eq:problem}
\sum_{m\in[M], t\in [T]}f_t^m(x_t^m) - \min_{\norm{x^\star}\leq B} \sum_{m\in[M], t\in [T]}f_t^m(x^\star),
\end{align} 
where $f_t^m$ is an arbitrary convex cost function observed by machine $m$ at time $t$, $x_t^m$ is the model the machine plays based on available history, and the comparator $x^\star$ is shared across the machines. This formulation is a natural extension of classical federated optimization \citep{mcmahan2016communication, kairouz2019advances}, but sequential decision-making and potentially adversarial cost functions bring new challenges to the problem. In particular, as opposed to usual federated optimization, where we want to come up with one good final model, the goal is to develop a sequence of instantaneously good models. This setup captures many real-world applications, such as mobile keyboard prediction \citep{hard2018federated, chen2019federated, google_ai_blog_2021}, self-driving vehicles \citep{elbir2020federated, nguyen2022deep}, voice assistants \citep{hao2020apple,googleVoice}, and recommendation systems \citep{shi2021federated,liang2021fedrec++,khan2021payload}. All these applications involve sequential decision-making across multiple machines where data is generated in real-time, and it might not be stored for future use due to memory or privacy concerns. Furthermore, these services must continually improve while deployed, so all the played models must be reasonably good. These challenges require designing new methods to benefit from collaboration while attaining small regret.

One approach is to solve the regret minimization problem \eqref{eq:problem} in a federated manner, i.e., by storing and analyzing the data locally at each machine while only communicating the models intermittently. This reduces the communication load and also mitigates some privacy concerns.  To capture this, we consider the \textit{intermittent communication} (IC) setting \citep{woodworth2018graph, woodworth2021min} where the machines work in parallel and are allowed to communicate $R$ times with $K$ time steps in between communication rounds. The IC setting has been widely studied over the past decade \citep{zinkevich2010parallelized, cotter2011better, dekel2012optimal, zhang2013information, zhang2015divide, shamir2014communication, stich2018local, dieuleveut2019communication, woodworth2020local, bullins2021stochastic, patel2022towards}, and it captures the expensive nature of communication in collaborative learning, such as in cross-device federated learning \citep{mcmahan2016communication, kairouz2019advances}.

Although many recent attempts \citep{Wang2020Distributed,dubey2020differentially,huang2021federated,li2022asynchronous,he2022simple,gauthier2022resource, gogineni2022communication, dai2022addressing, kuh2021real, mitra2021online} have been made towards tackling problem \eqref{eq:problem}, most existing works \citep{Wang2020Distributed,dubey2020differentially,huang2021federated,li2022asynchronous} study problem \eqref{eq:problem} by focusing on the \textbf{\textit{``stochastic"}} setting, where $\{f_t^m\}$'s are sampled from distributions specified in advance. However, real-world applications may have distribution shifts, un-modeled perturbations, or even an adversarial sequence of cost functions, all of which violate the fixed distribution assumption, and thus the above problem setups fail. To alleviate this issue, in this paper, we extend our understanding of distributed online and bandit convex optimization, i.e., problem \eqref{eq:problem}, to \textbf{\textit{``adaptive"}} adversaries that could potentially generate a worst-case sequence of cost functions. Although some recent works have underlined the importance of the adaptive setting \citep{gauthier2022resource, gogineni2022communication, dai2022addressing, kuh2021real, mitra2021online, he2022simple}, our understanding of the optimal regret guarantees for problem \eqref{eq:problem} with different classes of algorithms is still lacking.

\paragraph{\textbf{Our Contributions. }} Our work is the first attempt toward a systematic understanding of federated online optimization with limited feedback. We attain tight regret bounds (in some regimes) in the intermittent communication setting for first and zeroth-order feedback. We also bridge the gap between stochastic and adaptive settings in federated online optimization.
The following summarizes the main contributions of our work:
\begin{itemize}
    \item \textbf{Federated Online  Optimization with First-Order (Gradient) Feedback. } We first show in Section \ref{sec:first} that, under usual assumptions, there is no benefit of collaboration if all the machines have access to the gradients, a.k.a. first-order feedback for their cost functions. Specifically, in this setting, running online gradient descent on each device without any communication is min-max optimal for problem \eqref{eq:problem} (c.f., Theorems \ref{thm:first_lip}, \ref{thm:first_smth}). This motivates us to study weaker forms of feedback/oracles accesses where collaboration can be provably useful. 

    \item \textbf{Federated  Online Optimization with Zeroth-Order (Bandit) Feedback. } We then study the problem of federated adversarial linear bandits in Section \ref{sec:falb}, which is an important application of problem \eqref{eq:problem} and has not been studied in federated learning literature. We propose a novel federated projected online stochastic gradient descent algorithm to solve this problem (c.f., Algorithm \ref{alg:fed_pogd}). Our algorithm only requires single bandit feedback for every function. We show that when the problem dimension $d$ is large, our algorithm can outperform the non-collaborative baseline, where each agent solves its problem independently. Additionally, \textbf{when $\boldsymbol{d}$ is larger than $\boldsymbol{\ooo\big(M\sqrt{T/R}\big)}$, we prove that the proposed algorithm can achieve $\boldsymbol{\ooo\big(d/\sqrt{MT}\big)}$ average regret}, where $M$ is the number of agents and $R$ is the communication round (c.f., Theorem \ref{thm:favlb}). These results suggest that one can benefit from collaborations in the more challenging setting of using bandit feedback when the problem dimension is large.
    
    \item \textbf{Bandit Feedback Beyond Linearity. } We next consider the general (non-linear) problem \eqref{eq:problem} with bandit feedback in Section \ref{sec:2falb}, and study a natural variant of \textsc{FedAvg} equipped with a stochastic gradient estimator using two-point feedback \citep{shamir2017optimal}. We show that collaboration reduces the variance of the stochastic gradient estimator and is thus beneficial for problems of high enough dimension (c.f., Theorems \ref{thm:bd_grad_first_stoch} and \ref{thm:smooth_first_stoch}). We prove a linear speedup in the number of machines for high-dimensional problems, mimicking the stochastic setting with first-order gradient information \citep{woodworth2021min, woodworth2021even}. 
    
    \item \textbf{Tighter Rates with Two-Point Bandit Feedback.} Additionally, we show that (c.f., Corollary \ref{coro:linear}) one can achieve better regret for federated adversarial linear bandits using two-point feedback, indicating that multi-point feedback can be beneficial for federated adversarial linear bandits.

    \item \textbf{Characterizing the General Class of Problem. } Finally, we characterize the full space of related min-max problems in distributed online optimization, thus connecting the adversarial and stochastic versions of problem \eqref{eq:problem} (see Section \ref{sec:related}). Despite over a decade of work, this underlines how we understand only a small space of problems in distributed online optimization. This also identifies problems at the intersection of sequential and collaborative decision-making for future research.
\end{itemize}

\section{Problem Setting}\label{sec:setting}
 This section introduces definitions and assumptions used in our analysis. We formalize our adversary class and algorithm class. We also specify the notion of min-max regret.

\paragraph{\textbf{Notation.}} We denote the horizon by $T=KR$.  $\succeq, \preceq, \cong$ denote inequalities up to numerical constants. We denote the average function by $f_t(\cdot)~:=~\sum_{m\in[M]}f_t^m(\cdot)/M$ for all $t\in[T]$. We use $\mathbbm{1}_{A}$ to denote the indicator function for the event $A$. $\bb_2(B)\subset \rr^d$ denotes the $L_2$ ball of radius $B$ centered at zero. We suppress the indexing in the summation $\sum_{t\in[T], m\in[M]}$ to $\sum_{t,m}$ due to space constraints wherever the usage is clear. We will use $\{\cdot\}$ a bit loosely to denote both ordered and unordered sets depending on the usage.

\paragraph{\textbf{Function Classes.}} We consider several common function classes \citep{saha2011improved, hazan2016introduction} in this paper:
\begin{enumerate}
    \item $\boldsymbol{\fff^{G}}$, the class of convex, differentiable, and $G$-Lipschitz functions, i.e., $$\forall x,y\in\rr^d,\ |f(x)-f(y)|\leq G\norm{x-y}.$$
    \item $\boldsymbol{\fff^{H}}$, the class of convex, differentiable, and $H$-smooth functions, i.e., $$\forall x,y\in\rr^d,\ \norm{\nabla f(x)-\nabla f(y)}\leq H\norm{x-y}.$$ We also define $\boldsymbol{\fff^{G, H}} := \boldsymbol{\fff^{G}}\cap \boldsymbol{\fff^{H}}$, i.e., the class of Lipschitz and smooth functions.
    \item $\boldsymbol{\fff_{lin}^{G}} \subset \boldsymbol{\fff^{G, H}}$, which includes linear cost functions with gradients bounded by $G$. This includes the cost functions in federated adversarial linear bandits, discussed in Section \ref{sec:falb}. Note that linear functions are the ``smoothest" functions in the class $\boldsymbol{\fff^{G, H}}$ as they have a smoothness of $H=0$.
\end{enumerate}

\paragraph{\textbf{Adversary Model.}} Note that in the most general setting, each machine will encounter arbitrary functions from a class $\fff$ at each time step. Our algorithmic results are for this general model, which is usually referred to as an \textbf{\textit{``adaptive"}} adversary. Specifically, we allow the adversary's functions to depend on the past sequence of models played by the machines but not on the randomness used by the machines. Formally, consider the following sequence of histories for all time steps $t\in [T]$,
\begin{align*}
    \hhh_t &:= \cb{\{x^n_l\}^{n\in[M]}_{l\in[t-1]}, \{f^n_l\}^{n\in[M]}_{l\in[t-1]}}.
\end{align*}
Denote the adversary by $P\in \ppp$, where $\ppp$ defines some adversary class. We say that $P\rb{\hhh_t, A} = \eee_t \in \Delta(\fff^{\otimes M})$, i.e., the adversary looks at the history and the description of the algorithm $A\in \aaa$ and outputs a distribution on $M$ functions at time $t$, such that $\{f_t^m\}^{m\in[M]}\sim \eee_t$. We will discuss in Section \ref{sec:related} when randomization helps the adversary, v/s when it does not.

We also consider a weaker \textbf{\textit{``stochastic"}}\footnote{Note that we use stochastic to refer to \textit{``stochastic oblivious"} adversaries and not just any adversary that uses randomization, because as we saw the adaptive adversary might also randomize.} adversary model to explain some baseline results. More specifically, the adversary cannot adapt to the sequence of the models used by each machine but must fix a joint distribution in advance for sampling at each time step, i.e., for all $t\in[T]$, $\{f_t^m\}^{m\in[M]}\sim \eee \in \Delta(\fff^{\otimes M})$. We will denote the class of stochastic adversaries by $\ppp_{stoc}$. An example of this less challenging problem is \textit{distributed stochastic optimization} where $f_t^m(\cdot):= l(\cdot; z_t^m\sim \ddd_m)$ for some loss function $l$ such that $l(\cdot;z_t^m)\in \fff$ and $z_t^m$ is a data-point sampled from the distribution $\ddd_m$. With a slight abuse of notation, if we suppress $z\sim \ddd_m$ we can denote this as $\{f_t^m\}^{m\in[M]}\sim \eee = \ddd_1\times\dots\times\ddd_M$, i.e., $\eee$ is a product distribution with fixed distributions on each machine. For stochastic adversaries, we also denote $$F_m(x):= \ee_{f\sim \ddd_m}[f(x)],\ F(x):= \frac{1}{M}\sum_{m\in[M]}F_m(x).$$ The distinction between different adversary models is discussed further in Section \ref{sec:related}. We also consider the following assumption, which controls how similar the cost functions look across machines, thus limiting the adversary's power. 
\begin{assumption}[Bounded First-Order Heterogeneity]
\label{ass:heterogeneity}
For all $t\in [T],\ x\in \rr^d$, 
$$\frac{1}{M}\sum_{m\in[M]}\norm{\nabla f_t^m(x)- \nabla f_t(x)}^2 \leq \zeta^2\leq 4G^2.$$
\end{assumption}

\begin{remark}[Bounded First-Order Heterogeneity for Stochastic Setting]\label{rem:heterogeneity}
    In the stochastic setting, \citet{woodworth2020minibatch} consider a related but more relaxed assumption, i.e.,  $\forall\ x\in\rr^d$, $$\frac{1}{M}\sum_{m\in[M]}\norm{\nabla F_m(x) - \nabla F(x)}^2 \leq \zeta^2\leq 4G^2.$$ This assumption does not require the gradients at each time step to be point-wise close but requires them to be close only on average over the draws from the machines' distributions. Note that Assumption \ref{ass:heterogeneity} in the heterogeneous setting would be more akin to controlling the variance of distributions $\ddd_m$'s supported on ``similar" functions. In particular, in the stochastic homogeneous setting, when $\ddd_m=\ddd$, we note that $\zeta=0$ would imply that the cost function is fixed across time and machines. This is a less interesting problem, and many works in federated optimization (such as \citet{woodworth2020minibatch}) usually consider the more relaxed assumption stated above. 
\end{remark}

Defining the following quantity would also be useful for smooth functions.
\begin{definition}[Average Value at Optima]
\label{ass:loss_optimal}
For all $x^\star\in \arg\min_{x\in \bb_2(B)}\sum_{t\in [T]}f_t(x),$ $$\frac{1}{T}\sum_{t\in[T]}f_t(x^\star)~=~F_\star.$$ For non-negative functions in $\fff^{H}$, this implies, $\frac{1}{T}\sum_{t\in[T]}\norm{\nabla f_t(x^\star)}^2~\leq~HF_\star,$ (c.f., Lemma 4.1 in \citet{srebro2010optimistic}). $F_\star$ can be defined retrospectively once all the functions have been revealed.
\end{definition}

\begin{remark}[Bounded Optimality for Stochastic Setting]\label{rem:loss_optimal}
When the functions are generated stochastically, denote $F_\star:= \min_{x\in \bb_2(B)}F(x)$ as the minimum realizable function value. Then assuming $\ddd_m$ is supported on functions in $\fff^{H}$, implies for all $t\in[T]$, $$\ee_{\{f_t^m\sim \ddd_m\}_{m\in[M]}}\norm{\frac{1}{M}\sum_{m\in[M]}\nabla f_t^m(x^\star)}^2~\leq~HF_\star$$ (c.f., Lemma 3 in \citet{woodworth2021even}).
\end{remark}

We will now define our adversary classes. We use $\ppp\rb{\fff}$ to denote an adversary that selects functions from a function class $\fff$. We use argument $\zeta$ to restrict the adversary's selections using Assumption \ref{ass:heterogeneity}. In particular given a history $\hhh_t$, and $P\in \ppp\rb{\fff, \zeta}$, $P(\hhh_t)$ outputs a distribution $\eee_t$ which is supported on sets of functions in $\fff^{G\otimes M}$ which satisfy Assumption \ref{ass:heterogeneity}. The larger the $\zeta$ is, the more powerful the adversary is because its choices are less restricted. When $\zeta \geq 2G$, the adversary can select any $M$ functions from $\fff^G$. On the other hand, when $\zeta=0$, the adversary must pick the same cost functions for each machine up to constant terms. When the adversary is (oblivious) stochastic, we denote it is using a sub-script as $\ppp_{stoc.}\rb{\fff, \zeta}$. In this paper, we will consider four adversarial classes: $\ppp\rb{\fff^{G}, \zeta},\ \ppp\rb{\fff^{H}, \zeta},\ \ppp\rb{\fff^{G, H}, \zeta}, \text{ and } \ppp\rb{\fff_{lin}^{G}, \zeta}$. 

\paragraph{\textbf{Oracle Model.}} We consider three kinds of access to the cost functions in this paper. Each machine $m\in[M]$ for all time steps $t\in[T]$ has access to one of the following:
\begin{enumerate}
\item \textbf{Gradient} of $f_t^m$ at a single point, a.k.a., first-order feedback.
\item \textbf{Function value} of $f_t^m$ at a single point, a.k.a., one-point bandit feedback.
\item \textbf{Function values} of $f_t^m$ at two different points, a.k.a., two-point bandit feedback. 
\end{enumerate}
The Oracle model formally captures how the agent interacts with the cost functions. We will denote the Oracle call outputs by $O_t^m$ for machine $m$ at time $t$. For example, if the machine queries a two-point bandit oracle at points $x_t^{m,1}$ and $x_t^{m,2}$, then $O_t^m = (f_t^m(x_t^{m,1}), f_t^m(x_t^{m,2}))$. Note that we always look at the regret at the queried points. Thus in the two-point feedback setting, if the machine $m$ queries at points $x_t^{m,1}$ and $x_t^{m,2}$ at time $t$, it incurs the cost $f_t^m(x_t^{m,1}) + f_t^m(x_t^{m,2})$ (c.f., Theorem \ref{thm:bd_grad_first_stoch}).

\paragraph{\textbf{Algorithm Class.}} We assume the algorithms on each machine can depend on all the history it has seen. We define for all time steps $t\in [T]$, and $m\in[M]$,
\begin{align*}
    \ggg_t^m &:= \cb{\{x^n_l\}^{n\in[M]}_{l\in[\tau(t-1)]}, \{O^n_l\}^{n\in[M]}_{l\in[\tau(t-1)]}, \{x^m_l\}_{l\in[t-1]}, \{O^m_l\}_{l\in[t-1]}},
\end{align*}
where $\tau(t)$ denotes the last time step before or equal to $t$ when communication happened, i.e., $\tau(t) = t - t \text{ mod } K$. This captures the information structure of the  IC setting, i.e., the algorithms' output on a machine can depend on all the information (gradients or function values) it has seen on the machine or other machines' information communicated to it. We denote this class of algorithms by $\aaa_{IC}$ and add super-scripts $1$, $0$, $(0,2)$ to denote first-order, one-point bandit, and two-point bandit feedback. Thus, we consider three algorithm classes $\aaa_{IC}^1,\ \aaa_{IC}^0,\ \aaa_{IC}^{0,2}$ in this paper. Formally when $\{A_m\}_{m\in[M]}\in \aaa_{IC}^1$ or $\{A_m\}_{m\in[M]}\in \aaa_{IC}^0$, we say that $A_m(\ggg_t^m) = X_t^m \in \Delta(\rr^d)$ for all $m\in[M]$. Otherwise, when $\{A_m\}_{m\in[M]}\in \aaa_{IC}^{0,2}$ we say that $A_m(\ggg_t^m) = (X_t^{m,1}, X_t^{m,2})\in \Delta(\rr^d\times \rr^d)$ for all $m\in[M]$. In both these cases, we use $X_{\cdot}^{\cdot}$ to denote the randomization in the model $x_{\cdot}^{\cdot}$ played by the machine. We will notice below that randomization is crucial in bandit-feedback settings.

\paragraph{\textbf{Min-Max Regret.}} We are now ready to define our min-max regret formally. Let $\ppp$ be an adversary class and $\aaa$ be the algorithm class with single-point feedback. Then we define the min-max regret as follows,  
\begin{align}\label{eq:P2}
    \rrr(\ppp, \aaa) := \min_{A\in\aaa} \max_{P\in \ppp} \frac{1}{MT}\ee_{P,A}\rb{\sum_{t\in [T], m\in[M]}f_t^m(x_t^m)- \min_{x^\star\in \bb_2(B)}\sum_{t\in [T], m\in[M]}f_t^m(x^\star)},
\end{align}
where the expectation for the algorithm is w.r.t., the randomness in sampling $\{x_t^m\sim A(\ggg_t^m)\}_{t\in[T]}^{m\in[M]}$ and for the adversary it is w.r.t., the randomness in sampling $\{\{f_{t^m}\}^{m\in[M]}\sim \ppp(\hhh_t)\}_{t\in[T]}$. When there is two-point feedback involved, the min-max regret can be stated as,
\begin{align*}
    \rrr(\ppp, \aaa) := \min_{A\in\aaa} \max_{P\in \ppp} \frac{1}{2MT}\ee_{P,A}\rb{\sum_{t\in [T], m\in[M], j\in[2]}f_t^m(x_t^{m,j})- 2\cdot\min_{x^\star\in \bb_2(B)}\sum_{t\in [T], m\in[M]}f_t^m(x^\star)},
\end{align*}
and the expectation for the algorithm is w.r.t., the randomness in sampling $\{(x_t^{m,1}, x_t^{m,2})\sim A(\ggg_t^m)\}_{t\in[T]}^{m\in[M]}$. We want to characterize up to numerical constants the above min-max regrets for different adversary and algorithm classes described in this section. We sometimes denote the min-max regret as $\rrr_{M, K, R}(\ppp, \aaa)$ to highlight the number of machines, communication gap, and communication rounds, especially when we want to talk about the serial setting, $\rrr_{1, K, R}(\ppp, \aaa)$. Note that in this game, the second player, i.e., the adversary selecting the functions, does not benefit from randomization for problem \eqref{eq:P2}, which is why we can ignore the expectation w.r.t., $P$ in the above definitions and only look at deterministic adversarial classes $\ppp$ \footnote{It makes sense to make the randomization on the second player explicit when comparing to a weaker comparator, c.f., problem (P3) in Section \ref{sec:related}.}. It is also important to note that the max player does not have access to the randomness of the min player (c.f., problem (P1) in Section \ref{sec:related}). 

\paragraph{\textbf{Outline of the paper.}} We can now outline the results in this paper as follows,
\begin{itemize}
    \item In section \ref{sec:related}, we show how the problem studied in this paper and underlined in \eqref{eq:P2} is related to other problems with stochastic and adaptive adversaries. This will simplify the lower bound discussion in sections \ref{sec:first}, \ref{sec:falb} and \ref{sec:2falb}. And more importantly, it will put the adaptive and stochastic federated learning problems into perspective (c.f., figure \ref{fig:hierarchy}). 
    \item In section \ref{sec:first} we show that with first-order feedback the min-max regrets, $\rrr(\ppp(\fff^G, \zeta), \aaa_{IC}^1)$, and, $\rrr(\ppp(\fff^{H}, \zeta), \aaa_{IC}^1)$ are attained by independently running online gradient descent \citep{zinkevich2010parallelized} on each machine. This will imply that collaboration doesn't help in the presence of first-order feedback. Furthermore, since the optimal algorithm runs independently on each machine, there is no difference between large or small $\zeta$. It is an open question if we can improve with collaboration using any algorithm $\aaa_{IC}^1$ against an adversary from $\ppp(\fff^{G, H}, \zeta)$.
    \item In section \ref{sec:falb} we provide a novel algorithm in $\aaa_{IC}^0$, which benefits from collaboration against an adversary from  $\ppp(\fff_{lin}^{G}, \zeta)$ (in a high-dimensional regime). It is an open question to extend our algorithm to other adversaries, $\ppp(\fff^{G}, \zeta)$ or $\ppp(\fff^{G, H}, \zeta)$.
    \item In section \ref{sec:2falb} we provide a novel algorithm in $\aaa_{IC}^{0,2}$, which benefits from collaboration against adversaries from $\ppp(\fff_{lin}^{G}, \zeta)$, $\ppp(\fff^{G, H}, \zeta)$, as well as $\ppp(\fff^{G}, \zeta)$. We show that the algorithm is min-max optimal in several regimes. Furthermore, in the linear setting, the algorithm's guarantee improves with a smaller $\zeta$ as opposed to the single-feedback setting.   
\end{itemize}

\section{Related Problem Settings and Reductions}\label{sec:related}
\begin{figure}[!tbh]
    \centering
    \includegraphics[width=0.7\textwidth]{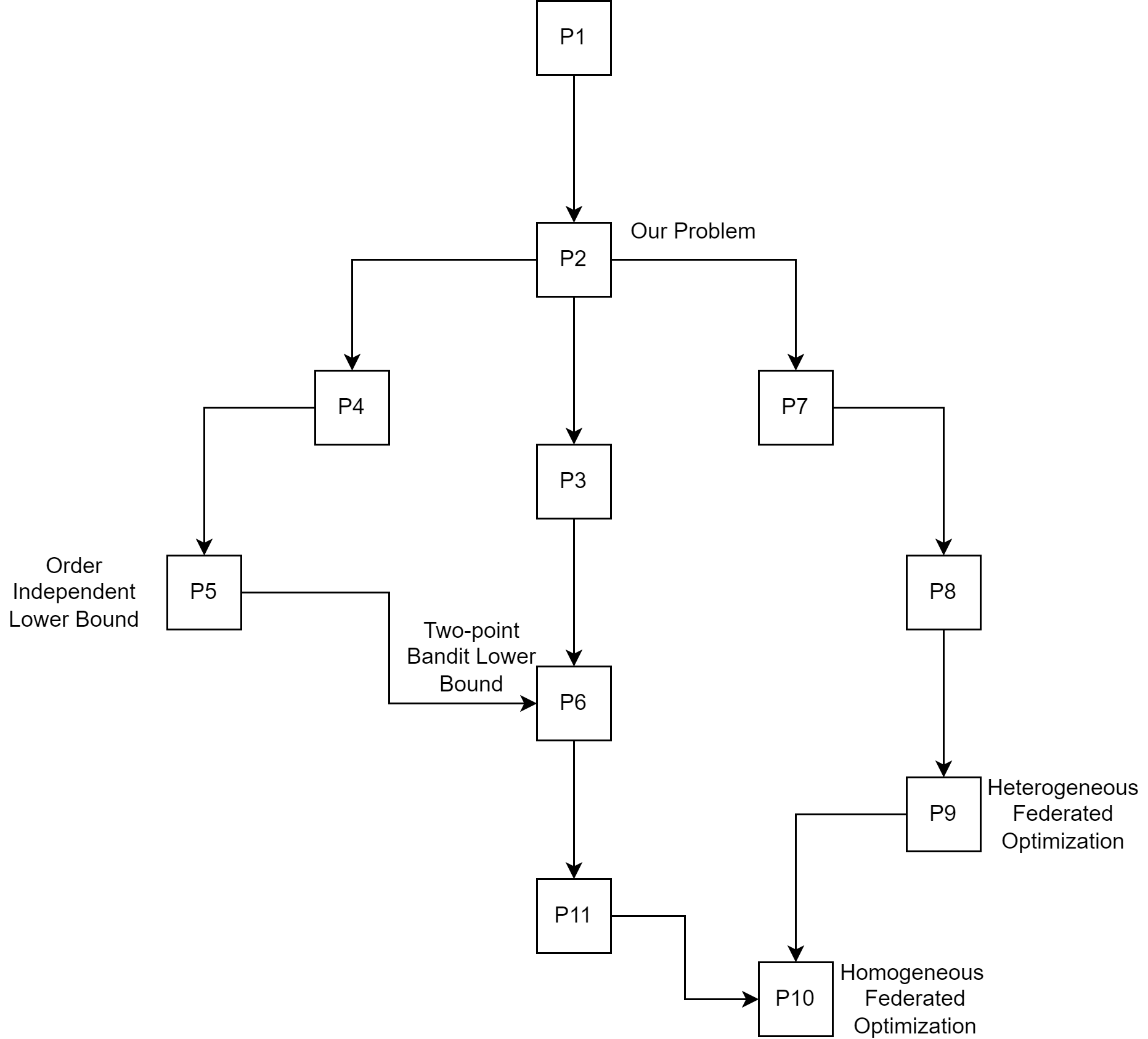}
    \caption{Summary of the problem space of federated online optimization. An arrow from the parent to child denotes that the child min-max problem is easier or has a lower min-max value than the parent problem. Note that to show that there is no benefit of collaboration for first-order algorithms for the problem (P2), we use the lower bound construction for the problem (P5). The figure clarifies why this does not contradict the benefit of collaboration for problems (P9) and (P10), as they lie on a different path from the parent (P2).}
    \label{fig:hierarchy}
\end{figure}
In this section, we will characterize different federated learning problems/objectives using the min-max value as we do in \eqref{eq:P2}. As before, we look at an algorithm class $\aaa$ and adversary class $\ppp$. The hardest problem we can hope to solve is when the adversary, besides knowing $\hhh_t$ and $A$, also knows the randomization of the algorithm at any given time $t$:
\begin{align}
    \min_{A\in\aaa} \ee_{A}\sb{\max_{P\in \ppp}  \ee_P\sb{\frac{1}{MT}\sum_{t\in [T], m\in[M]}f_t^m(x_t^m) - \min_{x^\star\in \bb_2(B)}\frac{1}{MT}\sum_{t\in [T], m\in[M]}f_t^m(x^\star)}}.\tag{\textbf{P1}}
\end{align}
For this problem, note that both the min and max-player do not gain from randomization, as on both players, respectively, we can just put all the distributions' mass on the best deterministic strategies. As a result, we can instead look at the sub-classes $\aaa_{det}\subset \aaa$, $\ppp_{det}\subset \ppp$ denoting deterministic strategies. This simplifies problem (P1) to the following:
\begin{align}
    \min_{A\in\aaa_{det}} \max_{P\in \ppp_{det}}  \frac{1}{MT}\sum_{t\in [T], m\in[M]}f_t^m(x_t^m) - \min_{x^\star\in \bb_2(B)}\frac{1}{MT}\sum_{t\in [T], m\in[M]}f_t^m(x^\star).\tag{\textbf{P1}}
\end{align}
Recall that this is not the problem we study in the paper, but instead, as defined in \eqref{eq:P2}, we look at the following simpler problem where the max-player/adversary does not know the random seeds of the min-player:
\begin{align}
    \min_{A\in\aaa}\max_{P\in \ppp}  \ee_{A, P}\sb{\frac{1}{MT}\sum_{t\in [T], m\in[M]}f_t^m(x_t^m) - \min_{x^\star\in \bb_2(B)}\frac{1}{MT}\sum_{t\in [T], m\in[M]}f_t^m(x^\star)}.\tag{\textbf{P2}}
\end{align}
Note that the adversary does not benefit from randomization, so we can replace the $\ppp$ above by $\ppp_{det}$. However, making the randomization on the max-player explicit makes it easier to state the following easier version of the problem (P2) with a weaker comparator $x^\star$ that does not depend on the randomness of the adversary and is thus worse in general:
\begin{align}
    \min_{A\in\aaa}\max_{P\in \ppp}  \ee_{A}\rb{\ee_P\sb{\frac{1}{MT}\sum_{t\in [T], m\in[M]}f_t^m(x_t^m)} - \min_{x^\star\in \bb_2(B)}\ee_P\sb{\frac{1}{MT}\sum_{t\in [T], m\in[M]}f_t^m(x^\star)}}. \tag{\textbf{P3}}
\end{align}
This form of regret is common in multi-armed bandit literature and is often called \textit{``pseudo-regret"}. Intuitively, one wants to disregard the random perturbations an adversary might add in multi-armed bandits while comparing to a hindsight optimal model. We have only discussed the \textit{``fully adaptive"} setting so far. We can also relax the problem (P2) by weakening the adversary. One way to do this is by requiring the functions to be the same across the machines, which leads to the following problem,
\begin{align*}
    \min_{A\in\aaa}\max_{P\in \ppp_{cood}}  \ee_{A}\sb{\frac{1}{MT}\sum_{t\in [T], m\in[M]}f_t(x_t^m) - \min_{x^\star\in \bb_2(B)}\frac{1}{T}\sum_{t\in [T]}f_t(x^\star)}.\tag{\textbf{P4}}
\end{align*}
By $\ppp_{cood}$, we denote the class of adversaries that make a coordinated attack. Such an attack will be useful when we show the lower bounds for algorithms in class $\aaa_{IC}^1$. Note that if the functions across the machines are the same, then $\zeta=0$ in Assumption \ref{ass:heterogeneity}. Depending on the algorithm class, this problem may or may not be equivalent to a fully serial problem, as we will show in this paper. We can further simplify problem (P4) by making the adversary stochastic,
\begin{align*}
    \min_{A\in\aaa}\max_{P\in \ppp_{cood, stoc}}  \ee_{A}\sb{\frac{1}{MT}\sum_{t\in [T], m\in[M]}f_t(x_t^m) - \min_{x^\star\in \bb_2(B)}\frac{1}{T}\sum_{t\in [T]}f_t(x^\star)}.\tag{\textbf{P5}}
\end{align*}
Since this amounts to just picking a fixed distribution on the cost functions that stays constant over time, we can alternatively restate problem (P5) in terms of choosing a distribution $\ddd\in \Delta(\fff)$ that can only depend on the description of the algorithm $A$,
\begin{align*}
    \min_{A\in\aaa}\max_{\ddd\in \Delta(\fff)}  \ee_{A, \{f_t\sim \ddd\}_{t\in[T]}}\sb{\frac{1}{MT}\sum_{t\in [T], m\in[M]}f_t(x_t^m) - \min_{x^\star\in \bb_2(B)}\frac{1}{T}\sum_{t\in [T]}f_t(x^\star)}. \tag{\textbf{P5}}
\end{align*}
We can further simplify problem (P5) by having a weaker comparator that does not depend on the randomness of sampling from $\ddd$, and by noting that the randomness of the adversary is independent of any randomness in the algorithm, 
\begin{align*}
    \min_{A\in\aaa}\max_{\ddd\in \Delta(\fff)}  \ee_{A, f_t\sim \ddd}\sb{\frac{1}{MT}\sum_{t\in [T], m\in[M]}\ee_{f_t\sim \ddd}\sb{f_t(x_t^m)}}- \min_{x^\star\in \bb_2(B)}\ee_{f\sim \ddd}\sb{f(x^\star)}. \tag{\textbf{P6}}
\end{align*}
Above, we have not removed the randomness for sampling the function in the expectation because the choice of the models $x_t^m$'s will also depend on this randomness. This can also be seen as a relaxation of the problem (P3) to have a stochastic adversary. Recalling the definition of $\{F_m := \ee_{f\sim \ddd_m}[f]\}_{m\in[M]}$ and $F:= \frac{1}{M}\sum_{m\in[M]}F_m$, and applying tower rule problem (P6) can be re-written as,
\begin{align*}
    \min_{A\in\aaa}\max_{\ddd\in \Delta(\fff)}  \ee_{A, f_t\sim \ddd}\sb{\frac{1}{MT}\sum_{t\in [T], m\in[M]}F(x_t^m)}- \min_{x^\star\in \bb_2(B)}F(x^\star). \tag{\textbf{P6}}
\end{align*}
Now let's relax (P2) directly to have stochastic adversaries which sample independently on each machine. In particular, this means there are fixed distributions $\ddd_m$ on each machine and at each time step $\{f_t^m\}^{m\in[M]}\sim \eee = \ddd_1\times\dots\times\ddd_m$. Note that this will require appropriately changing the problem classes' assumptions, as discussed in the remarks \ref{rem:heterogeneity} and \ref{rem:loss_optimal}.  To simplify the discussion, we assume that the problem class has no additional assumption and is just a selection of $MKR$ functions from some class $\fff$. This simplification lets us relax (P2) by picking the functions at machine $m$ at time $t$ from the distribution $\ddd_m$. This leads to the following problem,
\begin{align}
    \min_{A\in\aaa}\max_{\{\ddd_m\sim \Delta(\fff)\}_{m\in [M]}}  \ee_{A, \{f_t^m\sim \ddd_m\}_{t\in[T]}^{m\in [M]}}\sb{\frac{1}{MT}\sum_{t\in [T], m\in[M]}f_t^m(x_t^m) - \min_{x^\star\in \bb_2(B)}\frac{1}{MT}\sum_{t\in [T], m\in[M]}f_t^m(x^\star)}. \tag{\textbf{P7}}
\end{align}
If we weaken the comparator for this problem by not allowing it to depend on the randomness of sampling functions and recall the definitions for $F_m$ and $F$, we get the following problem,
\begin{align}
    \min_{A\in\aaa}\max_{\{\ddd_m\sim \Delta(\fff)\}_{m\in [M]}}  \ee_{A, \{f_t^m\sim \ddd_m\}_{t\in[T]}^{m\in [M]}}\sb{\frac{1}{MT}\sum_{t\in [T], m\in[M]}F_m(x_t^m)} - \min_{x^\star\in \bb_2(B)}F(x^\star). \tag{\textbf{P8}}
\end{align}
We note that problem (P8) is the regret minimization version of the usual heterogeneous federated optimization problem \citep{mcmahan2016communication, woodworth2020minibatch}. To make the final connection to the usual federated optimization literature, we note that the problem becomes easier if the algorithm can look at all the functions before deciding which model to choose. In other words, minimizing regret online is harder than obtaining one final retrospective model. This means we can simplify the problem (P8) to the following problem, where $A$ outputs $\hat{x}$ after looking at all the functions. More specifically, $A(\{\ggg_t^m\}_{t\in[T]}^{m\in[M]})=\hat{(X)}\in \Delta(\rr^d)$, and $\hat{x}\sim \hat X$. We denote the class of such algorithms by $\aaa_{opt}$, i.e., (stochastic) optimization algorithms. This allows us to relax to the following problem, 
\begin{align}
    \min_{A\in\aaa_{opt}}\max_{\{\ddd_m\sim \Delta(\fff)\}_{m\in [M]}}  \ee_{A, \{f_t^m\sim \ddd_m\}_{t\in[T]}^{m\in [M]}}\sb{\frac{1}{MT}\sum_{t\in [T], m\in[M]}F_m(\hat{x})} - \min_{x^\star\in \bb_2(B)}F(x^\star). \tag{\textbf{P9}}
\end{align}
With some re-writing of the notation, this reduces to the usual heterogeneous federated optimization problem \citep{mcmahan2016communication, woodworth2020minibatch, patel2022towards},
\begin{align}
    \min_{A\in\aaa}\max_{\{\ddd_m\sim \Delta(\fff)\}_{m\in [M]}}  \ee_{A, \{f_t^m\sim \ddd_m\}_{t\in[T]}^{m\in [M]}}\sb{F(\hat{x})} - \min_{x^\star\in \bb_2(B)}F(x^\star). \tag{\textbf{P9}}
\end{align}
Note that we don't remove the expectation w.r.t. sampling functions as $\hat x$ depends on that randomness, along with any randomness in $A$. Further assuming $\ddd_m= \ddd$ for all $m\in[M]$ (P9) to the usual homogeneous federated optimization problem \citep{woodworth2020local},
\begin{align*}
    \min_{A\in\aaa}\max_{\ddd\in \Delta(\fff)}  \ee_{A, \{f_t^m\sim \ddd\}^{m\in[M]}_{t\in[T]}}\sb{F(\hat{x})}- \min_{x^\star\in \bb_2(B)}F(x^\star). \tag{\textbf{P10}}
\end{align*}
Note that we can get a similar relaxation of the problem (P6) by converting regret minimization to find a final good solution. The problem will look as follows
\begin{align*}
    \min_{A\in\aaa}\max_{\ddd\in \Delta(\fff)}  \ee_{A, \{f_t\sim \ddd\}_{t\in[T]}}\sb{F(\hat{x})}- \min_{x^\star\in \bb_2(B)}F(x^\star). \tag{\textbf{P11}}
\end{align*}
The key difference between (P10) and (P11) is that $\hat{x}$ depends on $MT$ v/s $T$ random functions, respectively, in each case. This means (P10) is simpler than (P11) as it gets to see more information about the distribution $\ddd$. This concludes the discussion, and we summarize the comparisons between different min-max problems in Figure \ref{fig:hierarchy}. With this discussion, we are ready to understand the min-max complexities for the problem (P2) for different function and algorithm classes.

\section{Collaboration Does Not Help with First-Order Feedback}\label{sec:first}
This section will explain why collaboration does not improve the min-max regrets for the adaptive problem (P2) while using first-order oracles, even though collaboration helps for the stochastic problems (P9) and (P10). Specifically, we will characterize $\rrr(\ppp(\fff^{G}), \aaa_{IC}^1)$ and $\rrr(\ppp(\fff^{H}), \aaa_{IC}^1)$, where recall that $\aaa_{IC}^1$ is the class of algorithms with access to one gradient per cost function on each machine at each time step. This problem is well understood in the serial setting, i.e., when $M=1$. For instance, online gradient descent (OGD) \citep{zinkevich2010parallelized, hazan2016introduction} attains the min-max regret for both the adversary classes $\ppp(\fff^{G})$ and $\ppp(\fff^{H})$ \citep{woodworth2021even}. This raises the question of whether a distributed version of OGD can also be shown to be min-max optimal when $M>1$. The answer is, unfortunately, no! To state this more formally, we first introduce a trivial non-collaborative algorithm in the distributed setting: we run online gradient descent independently on each machine without any communication (see Algorithm \ref{alg:NC_OGD}). We prove the following result for this algorithm (which lies in the class $\aaa_{IC}^1$), essentially showing that the non-collaborative baseline is min-max optimal.     

{\LinesNumbered
\SetAlgoVlined
\begin{algorithm2e}[!thp]
	\caption{Non-collaborative \textsc{OGD} ($ \eta$)}\label{alg:NC_OGD}
		Initialize $x_0^m=0$ on all machines $m\in[M]$\\
            \For{$t\in\{0, \ldots, KR-1\}$}{
		      \For{$m \in [M]$ \textbf{in    parallel}}{
                    Play model $x_t^m$ and see function $f_t^m(\cdot)$\\
                    \textcolor{red}{\textbf{Incur loss} $f_t^m(x_t^m)$}\\
                    Compute gradient at point $x_t^m$ as $\nabla f_t^m(x_t^m)$ \\  
                    $x_{t+1}^m \gets x_{t}^m - \eta\cdot \nabla f_t^m(x_t^m)$
                }
		}	
\end{algorithm2e}

\begin{theorem}[Optimal Bounds with Non-collaborative \textsc{OGD} for Lipschitz Functions]
\label{thm:first_lip}
Algorithm \ref{alg:NC_OGD} incurs an average regret of $\ooo\rb{GB/\sqrt{T}}$ against the adversary class $\ppp(\fff^{G}, \zeta)$, where $B$ is the norm of the comparator. This regret is optimal against adversaries from $\ppp(\fff^{G, B}, \zeta)$ among algorithms in $\aaa_{IC}^1$, where $0\leq\zeta\leq 2G$, i.e., $\rrr(\ppp(\fff^{G}, \zeta), \aaa_{IC}^1) = \Theta(GB/\sqrt{T})$.  
\end{theorem}

\begin{proof}
We first prove the upper bound on the average regret of non-collaborative OGD and then show that it is optimal, i.e., equals 
$\rrr\rb{\ppp(\fff^{G}, \zeta), \aaa_{IC}^1}$. Note the following bound is always true for any stream of functions and sequence of models; we are just changing the comparator:
\begin{align*}
    \frac{1}{M}\sum_{m\in[M]}\rb{\sum_{t\in [T]}f_t^m(x_t^m) - \min_{x^{m, \star}\in \bb_2(B)}\sum_{t\in [T]}f_t^m(x^m)} \geq \frac{1}{M}\sum_{t\in [T], m\in[M]}f_t^m(x_t^m) - \min_{x^\star\in \bb_2(B)}\sum_{t\in [T]}f_t(x).
\end{align*}
This means we can upper bound $\rrr\rb{\ppp(\fff^{G}, \zeta), \aaa_{IC}^1}$ by running online gradient descent (OGD) independently on each machine without collaboration, i.e.,
\begin{align}\label{eq:GB1UB}
    \rrr_{M,K,R}\rb{\ppp(\fff^{G}, \zeta), \aaa_{IC}^1} \preceq \rrr_{1,K, R}\rb{\ppp(\fff^{G}), \aaa_{IC}^1} \cong \frac{GB}{\sqrt{T}}.
\end{align} 
The min-max rate for a single machine follows classical results using vanilla OGD \citep{zinkevich2010parallelized} (c.f., Theorem 3.1 by \citet{hazan2016introduction}). Now we prove that this average regret is optimal. Recall that we want to understand the problem (P2)'s lower bounds. Note that to lower bound (P2), we can lower bound any children problems in Figure \ref{fig:hierarchy}. 
In particular, from figure \ref{fig:hierarchy}, we can see that $(P2)\succeq (P4) \succeq (P5)$\footnote{We use $\preceq, \succeq$  to compare the problems by referring to their min-max regrets.} and then note that for the adversary in $(P5)$, $\zeta=0$ by design as all the machines see the same function. Furthermore, to lower bound problem (P5), we can lower bound the following quantity, as $\fff^G_{lin}\subset \fff^G$,
\begin{align*}
    \min_{A\in\aaa_{IC}^1}\max_{\ddd\in \Delta(\fff^{G}_{lin})}  \ee_{A, \{f_t\sim \ddd\}_{t\in[T]}}\sb{\frac{1}{MT}\sum_{t\in [T], m\in[M]}f_t(x_t^m) - \min_{x^\star\in \bb_2(B)}\frac{1}{T}\sum_{t\in [T]}f_t(x^\star)}.
\end{align*}
In other words it is sufficient to specify a distribution $\ddd\in \Delta(\fff^{G}_{lin})\subset \Delta(\fff^{G})$ such that for \textbf{any} sequence of models $\{x_t^m\}_{t\in[T]}^{m\in[M]}$, 
$$\ee_{\{f_t\sim \ddd\}_{t\in[T]}}\sb{\frac{1}{MT}\sum_{m,t}f_t(x_t^m) - \min_{x^\star\in \bb_2(B)}\frac{1}{T}\sum_{t}f_t(x^\star)}\succeq \frac{GB}{\sqrt{T}}.$$
Such lower bounds are folklore in serial online convex optimization (c.f., Theorem 3.2 \citep{hazan2016introduction}). One such easy construction is choosing $f_t(x) = \inner{\beta_t}{x}$ where $\beta_t\sim \frac{G}{\sqrt{d}}\cdot Unif(\{+1, -1\}^d)$. This ensures the following,
\begin{align}
    &\ee_{\{f_t\sim \ddd\}_{t\in[T]}}\sb{\frac{1}{MT}\sum_{m,t}f_t(x_t^m) - \min_{x^\star\in \bb_2(B)}\frac{1}{T}\sum_{t}f_t(x^\star)}\nonumber\\
    &= \ee_{\cb{\beta_t\sim \frac{G}{\sqrt{d}}\cdot Unif(\{+1, -1\}^d)}_{t\in[T]}}\sb{\frac{1}{MT}\sum_{m,t}\inner{\beta_t}{x_t^m} - \min_{x^\star\in \bb_2(B)}\frac{1}{T}\sum_{t}\inner{\beta_t}{x^\star}},\nonumber\\
    &= \ee_{\cb{\beta_t\sim \frac{G}{\sqrt{d}}\cdot Unif(\{+1, -1\}^d)}_{t\in[T]}}\sb{\frac{1}{MT}\sum_{m,t}\ee_{\beta_t\sim \frac{G}{\sqrt{d}}\cdot Unif(\{+1, -1\}^d)}\sb{\inner{\beta_t}{x_t^m}}}\nonumber\\ 
    &\qquad- \ee_{\{\beta_t\sim \frac{G}{\sqrt{d}}\cdot Unif(\{+1, -1\}^d)\}_{t\in[T]}}\sb{\min_{x^\star\in \bb_2(B)}\frac{1}{T}\sum_{t}\inner{\beta_t}{x^\star}},\nonumber\\
    &= 0 - \frac{GB}{Td}\ee_{\{\beta_t\sim \cdot Unif(\{+1, -1\}^d)\}_{t\in[T]}}\sb{\min_{x^\star\in \bb_2(\sqrt{d})}\inner{\sum_{t}\beta_t}{x^\star}},\nonumber\\
    &\geq \frac{GB}{Td}\sum_{i\in[d]}\ee_{\{\beta_{t,i}\sim \cdot Unif(\{+1, -1\})\}_{t\in[T]}}\sb{-\min_{|x^\star_i|\leq 1}\inner{\sum_{t}\beta_{t,i}}{x^\star_i}},\nonumber\\
    &= \frac{GB}{T}\ee_{\{u_t\sim \cdot Unif(\{+1, -1\})\}_{t\in[T]}}\sb{-\min_{|y^\star|\leq 1}\inner{\sum_{t}u_{t}}{y^\star}},\nonumber\\
    &= \frac{GB}{T}\ee_{\{u_t\sim \cdot Unif(\{+1, -1\})\}_{t\in[T]}}\sb{\lvert\sum_{t}u_{t}\rvert},\nonumber\\
    &\geq \frac{GB}{2\sqrt{T}} \label{eq:GB1LB},
\end{align}
where the first inequality uses the fact that splitting across the dimensions can only hurt the minimization, thus making the overall quantity smaller, and the last inequality uses a standard result about the absolute sum of Rademacher random variables \footnote{For instance, see \href{https://en.wikipedia.org/wiki/Random_walk\#One-dimensional_random_walk}{this standard result} on single dimensional random walks.}. This finishes the lower bound proof. We have thus shown that the regret of the non-collaborative baseline is optimal, and combining bounds \eqref{eq:GB1UB} and \eqref{eq:GB1LB}, we can conclude that $$\rrr(\ppp(\fff^{G}, \zeta), \aaa_{IC}^1) \cong \frac{GB}{\sqrt{T}}.$$ 
This completes the proof.
\end{proof}

As the lower bound instance is linear, it is easy to show that the above theorem  also characterizes $\rrr(\ppp(\fff^{G}_{lin}, \zeta), \aaa_{IC}^1)$. We can use the same proof strategy to prove the following result for smooth functions. 

\begin{theorem}[Optimal Bounds with Non-collaborative \textsc{OGD} for Smooth Functions]\label{thm:first_smth}
   Algorithm \ref{alg:NC_OGD} incurs the optimal regret of $\Theta\rb{HB^2/T + \sqrt{HF_\star}B/\sqrt{T}}$ against adversary class $\ppp(\fff^{H}, \zeta)$, where $B$ is the norm of the comparator and $F_\star$ is from definition \ref{ass:loss_optimal}.  
\end{theorem}
\begin{proof}
    For the upper bound, we can use the upper bound for online gradient descent in the serial setting following from a classical work on optimistic rates (c.f, Theorem 3 \cite{srebro2010optimistic}). Then we use the same lower-bounding strategy as in theorem \ref{thm:first_lip} but instead lower bound (P11), and note that $(P2)\succeq (P4) \succeq (P5) \succeq (P6) \succeq (P11)$. Focusing on (P11) makes $\zeta=0$ since our attack is coordinated. Then to lower bound (P11), we use the construction and distribution as used in the proof of Theorem 4 by \citet{woodworth2021even}, which is a sample complexity lower bound that only depends on $T$, i.e., the number of samples observed from $\ddd$. This finishes the proof.
\end{proof}

\textbf{Implications of Theorems \ref{thm:first_lip} and \ref{thm:first_smth}: }
The above theorems imply that there is no benefit of collaboration in the worst case if the machines already have access to gradient information! This is counter-intuitive at first because several works have shown in the stochastic setting collaboration indeed helps \citep{woodworth2020minibatch, koloskova2020unified}. 

\begin{center}
\textbf{How do we reconcile these results?} 
\end{center}

Note that while proving theorem \ref{thm:first_lip}, we crucially rely on the chain of reductions $(P2)\succeq (P4) \succeq (P5)$. Similarly, while proving theorem \ref{thm:first_smth}, we rely on the chain of reductions $(P2)\succeq (P4) \succeq (P5)\succeq (P6) \succeq (P11)$. These reductions allow us to lower-bound the min-max regret through an adversary that can use the same function on each machine. This is the main difference w.r.t. usual federated optimization literature, where the problems of interest are (P9) and (P10), and such coordinated attacks (making $\zeta=0$) are not possible for non-degenerate distributions. This becomes clear by looking at figure \ref{fig:hierarchy}, where (P5) and (P11) are both at least as hard as (P10), and (P9) is on a different chain of reductions. This means the lower bounds in theorems \ref{thm:first_lip} and \ref{thm:first_smth} do not apply to the usual stochastic federated optimization and that there is no contradiction. Another way to view the tree is that any lower problem in the tree does not necessarily suffer from the lower bounds that apply to its parents. Thus, (P10) is not limited by the lower bound applicable for (P11).

\begin{remark}
Also, note from the above theorems that having a first-order heterogeneity bound $\zeta$ does not help because $\zeta=0$ for the problem (P4). This is unsurprising as we used a coordinated attack to give the lower bounds. However, a small $\zeta$ should intuitively help in the stochastic federated settings, i.e., for problems (P9) and (P10), as it restricts the clients' distributions. Having said that, as discussed in remark \ref{rem:heterogeneity}, the heterogeneity assumption in the stochastic setting for the problem (P9) is a bit different.    
\end{remark}

\section{Collaboration Helps with Bandit Feedback}\label{sec:falb}
In this section, we move our attention to the more challenging setting of using bandit (zeroth-order) feedback at each machine. We begin by studying one important application of problem \eqref{eq:problem}, i.e., \emph{federated adversarial linear bandits}. We then investigate the \emph{more general problem of federated bandit convex optimization with two-point feedback} in the next section.

\textbf{Federated Adversarial Linear Bandits.}
One important application of problem \eqref{eq:problem} is federated linear bandits, which has received increasing attention in recent years. However, most existing works \citep{Wang2020Distributed,huang2021federated,li2022asynchronous,he2022simple} do not consider the more challenging adaptive adversaries, leaving it unclear whether collaboration can be beneficial in this scenario.  Therefore, we propose to study federated adversarial linear bandits, a natural extension of single-agent adversarial linear bandits \citep{bubeck2012regret} to the federated optimization setting. Specifically, at each time $t\in[T]$, an agent $m\in[M]$ chooses an action $x_t^m\in\rr^d$ while simultaneously environment picks $\beta_t^m\in \bb_2(G)\subset \rr^d$. Agent $m$ then suffers the loss $f_t^m(x_t^m)=\langle\beta_t^m,x_t^m\rangle$. Our goal is to output a sequence of models $\{x_t^m\}_{t\in[T]}^{m\in[M]}$ that minimize the following expected regret
\begin{align}
\ee\sb{\sum_{m, t}\inner{\beta_t^m}{x_t^m} - \min_{\norm{x^\star}\leq B}\sum_{m, t}\inner{\beta_t^m}{x^\star}},\label{eq:problem_falb}
\end{align}
where the expectation is w.r.t. the randomness of the model selection. To solve this federated adversarial linear bandits problem, we propose a novel projected online stochastic gradient descent with one-point feedback algorithm, which we call \textsc{FedPOSGD}, as illustrated in Algorithm \ref{alg:fed_pogd}.

{\LinesNumbered
\SetAlgoVlined
\begin{algorithm2e}[!t]
	\caption{\textsc{FedPOSGD} ($ \eta,\delta$) with one-point bandit feedback}\label{alg:fed_pogd}
		Initialize $x_0^m=0$ on all machines $m\in[M]$\\
            \For{$t\in\{0, \ldots, KR-1\}$}{
		      \For{$m \in [M]$ \textbf{in parallel}}{
                    $w_t^m=\textbf{Proj}(x_t^m)$\\
		        Sample $u_t^m\sim Unif(\ss_{d-1})$, i.e., a random unit vector\\          
                    Query function $f_t^m$ at point $w_t^{m,1}=w_t^{m} + \delta u_t^m$\\
                    \textcolor{red}{\textbf{Incur loss} $f_t^m(w_t^{m} + \delta u_t^m)$}\\
                    Compute stochastic gradient at point $w_t^m$ as $g_t^m = df(w_t^m+\delta u_t^m)u_t^m/\delta$ \\  
                    \If{$(t+1)\mod K =0$}{ 
                    \textcolor{blue}{\textbf{Communicate to server:}} $\rb{x_t^m - \eta\cdot g_t^m}$\\
                    On server $x_{t+1} \gets \frac{1}{M}\sum_{m\in[M]}\left(x_t^m - \eta\cdot g_t^m\right)$\\
                    \textcolor{blue}{\textbf{Communicate to machine:}} $x_{t+1}^m \gets x_{t+1}$}
                    \Else{
                    $x_{t+1}^m \gets x_{t}^m - \eta\cdot g_t^m$
                    }
                }
		}	
\end{algorithm2e}

At the core of Algorithm \ref{alg:fed_pogd} is the gradient estimator $g_t^m$ constructed using one-point feedback (see line 8). This approach is motivated by \citet{flaxman2004online}, with the key distinction that our gradient estimator is calculated at the projected point $w_t^m=\textbf{Proj}(x_t^m)=\arg\min_{\|w\|_2\leq B}\|w-x_t^m\|_2$ (see line 4). For the linear cost function $f_t^m(x)=\langle\beta_t^m,x\rangle$, we can show that $g_t^m$ is an unbiased gradient estimator $\ee_{u_t^m}[g_t^m]=\nabla f_t^m(x)$ with the variance \citep{hazan2016introduction}:
\begin{align*}
    \ee_{u_t^m}\sb{\norm{g_t^m-\nabla f_t^m(x)}^2}\preceq \big(d\|\beta_t^m\|_2\cdot(\|x\|_2+\delta)/\delta\big)^2.
\end{align*}
Therefore, the projection step in Algorithm \ref{alg:fed_pogd} can ensure the variance of our gradient estimator is bounded. However, the extra projection step can make it difficult to benefit from collaboration when we have gradient estimators from multiple agents. To address this issue, we propose to perform the gradient update in the unprojected space, i.e., $x_{t}^m - \eta\cdot g_t^m$ (see line 14), instead of the projected space, i.e., $\textbf{Proj}(w_{t}^m - \eta\cdot g_t^m)$, which is motivated by the lazy mirror descent based methods \citep{nesterov2009primal,bubeck2015convex,yuan2021federated}. We obtain the following guarantee for Algorithm \ref{alg:fed_pogd} for federated adversarial linear bandits.

\begin{theorem}[Regret Guarantee of Algorithm \ref{alg:fed_pogd} for Federated Adversarial Linear Bandits]
\label{thm:favlb}
Against the adversary class $\ppp(\fff_{lin}^{G},\zeta)$, if we choose $\eta = \frac{B}{G\sqrt{T}}\cdot\min\left\{1, \frac{\sqrt{M}}{dB}, \frac{1}{\mathbbm{1}_{K>1}\sqrt{d B}K^{1/4}},\frac{\sqrt{G}}{\mathbbm{1}_{K>1}\sqrt{\zeta K}}\right\}$ and $\delta=B$, the queried points $\{w_t^{m,1}\}_{t,m=1}^{T,M}$ of Algorithm \ref{alg:fed_pogd} satisfy:
\begin{align*}
    \frac{1}{MKR}\sum_{t\in [KR], m\in[M]}\ee\sb{f_t^m(w_t^{m,1})-f_t^m(x^\star)} \preceq  
        \frac{GB}{\sqrt{KR}} + \frac{GB d}{\sqrt{MKR}} + \mathbbm{1}_{K>1}\cdot\rb{\frac{GB\sqrt{d}}{K^{1/4}\sqrt{R}}+\frac{\sqrt{G\zeta}B}{\sqrt{R}}},
\end{align*}
    where $x^\star \in \arg\min_{x\in\bb_2(B)}\sum_{t\in [KR]}f_t(x)$, and the expectation is w.r.t. the choice of function queries.
\end{theorem}

\textbf{Implication of Theorem \ref{thm:favlb}: }
First, let us compare our results to the non-collaborative baseline as we did in Section \ref{sec:first}. More specifically, the non-collaborative baseline 
\footnote{For the non-collaborative baseline, we can run SCRIBLE \citep{hazan2016introduction} independently on each machine or run Algorithm \ref{alg:NC_OGD} using the gradient estimator proposed by \citet{flaxman2004online} with an extra projection step, i.e., $x^m_{t+1}=\textbf{Proj}(x^m_{t+1})$.} 
can obtain the average regret of $\ooo\rb{GBd/\sqrt{KR}}$. If we have $d\succeq \sqrt{K}$, \textsc{FedPOSGD} outperforms the non-collaborative baseline. Furthermore, if we have $d\succeq \sqrt{K}M$, then the average regret of \textsc{FedPOSGD} will be dominated by $\ooo\big(GBd/\sqrt{MKR}\big)$. As a result,  \textsc{FedPOSGD} achieves a \textit{``linear speed-up"} in terms of the number of machines compared to the non-collaborative baseline. Although a smaller $\zeta$ will lead to a smaller regret bound, we cannot benefit from a small $\zeta$. This is because the $\zeta$ dependent term $\sqrt{G\zeta}B/\sqrt{R}$ in the average regret bound will be dominated by other terms when the problem dimension is large, i.e., $d\succeq \sqrt{K}$.

\textbf{Limitations of Algorithm \ref{alg:fed_pogd}.} Based on the average regret, we can see that: (1) the proposed one-point feedback algorithm, i.e., \textsc{FedPOSGD}, requires \textit{an extra projection step}; (2) the averaged regret bound of \textsc{FedPOSGD} \textit{increases linearly with the problem dimension} $d$; and (3) the average regret of \textsc{FedPOSGD} \textit{cannot benefit from the small heterogeneity} in the regime where \textsc{FedPOSGD} outperforms the non-collaborative baseline. To address these issues, we propose to use a two-point feedback algorithm in the next section.

\section{Better Rates with Two-Point Bandit Feedback}\label{sec:2falb}
In this section, we study distributed bandit convex optimization with two-point feedback, i.e., at each time step, the machines can query the value (and not the full gradient) of their cost functions at two different points. We show an improved regret guarantee for general Lipschitz smooth functions and then specify the improvements for federated adversarial linear bandits. 

The two-point feedback structure is useful for single-agent bandit convex optimization, as it can help attain the optimal horizon dependence in the regret  \citep{duchi2015optimal, shamir2017optimal} using simple algorithms. We consider a general convex cost function rather than the linear cost function discussed in the last section. We analyze the online variant of the \textsc{FedAvg} or \textsc{Local-SGD} algorithm, which is commonly used in the stochastic setting. We refer to this algorithm as \textsc{FedOSGD} and describe it in Algorithm \ref{alg:fed_ogd}.
{\LinesNumbered
\SetAlgoVlined
\begin{algorithm2e}[!thp]
	\caption{\textsc{FedOSGD} ($ \eta, \delta$) with two-point bandit feedback}
    \label{alg:fed_ogd}
		Initialize $x_0^m=0$ on all machines $m\in[M]$\\
            \For{$t\in\{0, \ldots, KR-1\}$}{
		      \For{$m \in [M]$ \textbf{in parallel}}{
		        Sample $u_t^m\sim Unif(\ss_{d-1})$, i.e., a random unit vector\\          
                    Query function $f_t^m$ at points $(x_t^{m,1}, x_t^{m,2}):=(x_t^{m} + \delta u_t^m, x_t^{m} - \delta u_t^m)$\\
                    \textcolor{red}{\textbf{Incur loss} $(f_t^m(x_t^{m} + \delta u_t^m) + f_t^m(x_t^{m} - \delta u_t^m))$}\\
                    Compute stochastic gradient at point $x_t^m$ as $g_t^m = \frac{d(f(x_t^m+\delta u_t^m) - f(x_t^m-\delta u_t^m))u_t^m}{2\delta}$ \\  
                    \If{$(t+1)\mod K =0$}{ 
                    \textcolor{blue}{\textbf{Communicate to server:}} $\rb{x_t^m - \eta\cdot g_t^m}$\\
                    On server $x_{t+1} \gets \frac{1}{M}\sum_{m\in[M]}\left(x_t^m - \eta\cdot g_t^m\right)$\\
                    \textcolor{blue}{\textbf{Communicate to machine:}} $x_{t+1}^m \gets x_{t+1}$}
                    \Else{
                    $x_{t+1}^m \gets x_{t}^m - \eta\cdot g_t^m$
                    }
                }
		}	
\end{algorithm2e}

The key idea in \textsc{FedOSGD} is using the gradient estimator constructed with two-point feedback (see line 7), originally proposed by \citet{shamir2017optimal} and based on a similar estimator by \citet{duchi2015optimal}. For a smoothed version of the function $f_t^m$, i.e., $\hat{f}_t^m(x):=\ee_{u_t^m}[f_t^m(x+\delta u_t^m)]$, the gradient estimator $g_t^m$ is an unbiased estimator $\ee_{u_t^m}[g_t^m] = \nabla \hat{f}_t^m(x)$ with variance (c.f., Lemmas 3 and 5, \citep{shamir2017optimal}):
\begin{align*}
    \ee_{u_t^m}\sb{\norm{g_t^m-\nabla \hat{f}_t^m(x)}^2}\preceq dG^2,
\end{align*}
where $G$ is the Lipschitz parameter for $f_t^m$.

Equipped with this gradient estimator, we can prove the following guarantee for $\ppp_{M, K, R}(\fff^{G, B}, \zeta)$ using \textsc{FedOSGD}.

\begin{theorem}[Better Bounds with Two-Point Feedback for Lipschitz Functions]
\label{thm:bd_grad_first_stoch}
Against the adversary class class $\ppp(\fff^{G}, \zeta)$, if we choose $\eta = \frac{B}{G\sqrt{T}}\cdot\min\left\{1, \frac{\sqrt{M}}{\sqrt{d}}, \frac{1}{\mathbbm{1}_{K>1}\sqrt{K}d^{1/4}}\right\}$, and $\delta = \frac{Bd^{1/4}}{\sqrt{R}}\rb{1+ \frac{d^{1/4}}{\sqrt{MK}}}$, the queried points $\{x_t^{m,j}\}_{t,m,j=1}^{T,M,2}$ of Algorithm \ref{alg:fed_ogd} satisfy:
\begin{align*}
   \frac{1}{2MKR}\sum_{t\in [KR], m\in[M], j\in[2]}\ee\sb{f_t^m(x_t^{m,j})-f_t^m(x^\star)} \preceq  
        \frac{GB}{\sqrt{KR}} + \frac{GB\sqrt{d}}{\sqrt{MKR}} + \mathbbm{1}_{K>1}\cdot\frac{GBd^{1/4}}{\sqrt{R}},
\end{align*}
    where $x^\star \in \arg\min_{x\in\bb_2(B)}\sum_{t\in [KR]}f_t(x)$, and the expectation is w.r.t. the choice of function queries.
\end{theorem}

\textbf{Implication of Theorem \ref{thm:bd_grad_first_stoch}:}
When $K=1$, the above average regret reduces to the first two terms, which are known to be tight for two-point bandit feedback \citep{duchi2015optimal,hazan2016introduction} (see Appendix \ref{sec:lower_bounds}), making \textsc{FedOSGD} optimal. When $K>1$, we would like to compare our results to the non-collaborative baseline as we did in Section \ref{sec:first}. The non-collaborative baseline \footnote{For the non-collaborative baseline, we run Algorithm \ref{alg:NC_OGD} using the gradient estimator proposed by \citet{shamir2017optimal}.} attains the average regret of $\ooo\rb{GB\sqrt{d}/\sqrt{KR}}$. Thus, as long as $d\succeq K^2$, \textsc{FedOSGD} performs better than the non-collaborative baseline. Furthermore, if $d\succeq K^2M^2$, then the average regret of \textsc{FedOSGD} is dominated by $\ooo\big(GB\sqrt{d}/\sqrt{MKR}\big)$. Therefore, \textsc{FedOSGD} achieves a \textit{``linear speed-up"} in the number of machines compared to the non-collaborative baseline. Unfortunately, the bound doesn't improve with smaller $\zeta$.

Note that the Lipschitzness assumption is crucial for the two-point gradient estimator in Algorithm \ref{alg:fed_ogd} to control the variance of the proposed gradient estimator. While there are gradient estimators that do not require Lipschitzness or bounded gradients \citep{flaxman2004online}, they impose stronger assumptions such as bounded function values or necessitate extra projection steps (see the gradient estimator in Algorithm \ref{alg:fed_pogd}). To avoid making these assumptions, we instead focus on problems in $\ppp_{M, K, R}(\fff^{G, H}, \zeta)$  rather than problems in $\ppp_{M, K, R}(\fff^{H}, \zeta)$. 

\begin{theorem}[Better Bounds with Two-Point Feedback for Smooth Functions]
\label{thm:smooth_first_stoch}
Against the adversary class $\ppp(\fff^{G,H}, \zeta)$. If we choose appropriate $\eta, \delta$ (c.f., Lemma \ref{lem:smooth_first_stoch} in Appendix \ref{app:smth}), the queried points $\{x_t^{m,j}\}_{t,m,j=1}^{T,M,2}$ of Algorithm \ref{alg:fed_ogd} satisfy:
    \begin{align*}
       \frac{1}{2MKR}&\sum_{t\in [KR], m\in[M], j\in[2]}\ee\sb{f_t^m(x_t^{m,j})-f_t^m(x^\star)} \preceq 
        \frac{HB^2}{KR} + \frac{\sqrt{HF_\star}B}{\sqrt{KR}} + \frac{GB}{\sqrt{KR}}+ \frac{ GB\sqrt{d}}{\sqrt{MKR}} \\
        &+\mathbbm{1}_{K>1}\cdot\min\Bigg\{\frac{H^{1/3}B^{4/3}G^{2/3}d^{1/3}}{K^{1/3}R^{2/3}} + \frac{H^{1/3}B^{4/3}\zeta^{2/3}}{R^{2/3}} + \frac{\sqrt{\zeta G}Bd^{1/4}}{K^{1/4}\sqrt{R}} + \frac{\zeta B}{\sqrt{R}}, \frac{GBd^{1/4}}{K^{1/4}\sqrt{R}} + \frac{\sqrt{G\zeta}B}{\sqrt{R}}\Bigg\},    
    \end{align*}
    where $x^\star \in \arg\min_{x\in \bb_2(B)}\sum_{t\in [KR]}f_t(x)$, and the expectation is w.r.t. the choice of function queries. Furthermore, we can obtain the same average regret as in Theorem \ref{thm:bd_grad_first_stoch} with the corresponding step size.
\end{theorem}


The above result is a bit technical, so we consider the simpler class $\fff_{lin}^{G, B}$ of linear functions with bounded gradients to interpret it. Linear functions are the ``smoothest" Lipschitz functions as their smoothness constant $H=0$. We can get the following guarantee (see Appendix \ref{sec:modify}):
\begin{corollary}[Better Bounds with Two-Point Feedback for Linear Functions]
\label{coro:linear}
    Against the adversary class $\ppp(\fff^{G}_{lin}, \zeta)$, if we choose the same $\eta$ and $\delta$ as in Theorem \ref{thm:smooth_first_stoch} (with $H=0$), the queried points $\{x_t^{m,j}\}_{t,m,j=1}^{T,M,2}$ of Algorithm \ref{alg:fed_ogd} satisfy:
    \begin{align*}
       \frac{1}{2MKR}\sum_{t\in [T], m\in[M], j\in[2]}\ee\sb{f_t^m(x_t^{m,j})-f_t^m(x^\star)} &\preceq 
        \frac{GB}{\sqrt{KR}} + \frac{ GB\sqrt{d}}{\sqrt{MKR}} + \mathbbm{1}_{K>1}\cdot \rb{\frac{\sqrt{\zeta G}Bd^{1/4}}{K^{1/4}\sqrt{R}} + \frac{\zeta B}{\sqrt{R}}},  
    \end{align*}
    where $x^\star \in \arg\min_{x\in\bb_2(B)}\sum_{t\in [KR]}f_t(x)$, and the expectation is w.r.t. the choice of function queries.
\end{corollary}

\textbf{Implications of Theorem \ref{thm:smooth_first_stoch} and Corollary \ref{coro:linear}:}
Unlike general Lipschitz functions, the last two terms in the average regret bound for linear functions are zero when $\zeta=0$, and the upper bound is smaller for smaller $\zeta$. In fact, when $K=1$ or $\zeta=0$, the upper bound reduces to $\ooo\big(GB/\sqrt{KR}+GB\sqrt{d}/\sqrt{MKR}\big)$, which is optimal (c.f., Appendix \ref{sec:lower_bounds}). These results show that \textsc{FedOSGD} can benefit from small heterogeneity.
More generally, when $K\leq\max\big\{1, G^2\zeta^2d, G^2d/(\zeta^2M^2)\big\}$, then \textsc{FedOSGD} again achieves the optimal average regret of $\ooo\big(GB/\sqrt{KR}+GB\sqrt{d}/\sqrt{MKR}\big)$. Recall that in this setting, the non-collaborative baseline \footnote{The same baseline as for problems in $\ppp_{M,K,R}(\fff^{G,B}, \zeta)$.} obtains an average regret of $\ooo(GB\sqrt{d}/\sqrt{KR})$. Thus, the benefit of collaboration through \textsc{FedOSGD} again appears in high-dimensional problems in $\ppp_{M, K, R}(\fff^{G, B}_{lin}, \zeta, F_\star)$ similar to what we discussed for $\ppp_{M, K, R}(\fff^{G, B}, \zeta)$. 

\paragraph{\textbf{Single v/s Two-Point Feedback.}} For federated adversarial linear bandits, we can directly apply Algorithm \ref{alg:fed_ogd}, i.e., \textsc{FedOSGD}, to solve it and achieve the average regret bound in Corollary \ref{coro:linear} as follows:
\begin{align*}
    \frac{GB}{\sqrt{KR}} + \frac{ GB\sqrt{d}}{\sqrt{MKR}} + \mathbbm{1}_{K>1}\cdot \rb{\frac{\sqrt{\zeta G}Bd^{1/4}}{K^{1/4}\sqrt{R}} + \frac{\zeta B}{\sqrt{R}}}.
\end{align*}
Recall that we can also use the one-point feedback based Algorithm \ref{alg:fed_pogd}, i.e., \textsc{FedPOSGD}, to get the following average regret bound for federated adversarial linear bandits:
\begin{align*}
    \frac{GB}{\sqrt{KR}} + \frac{GB d}{\sqrt{MKR}} + \mathbbm{1}_{K>1}\cdot\bigg(\frac{GB\sqrt{d}}{K^{1/4}\sqrt{R}}+\frac{\sqrt{G\zeta}B}{\sqrt{R}}\bigg).
\end{align*}
According to these results, \textsc{FedOSGD} significantly improves the average regret bound of \textsc{FedPOSGD} regarding the dependence on $d$ and $\zeta$. In addition, \textsc{FedOSGD} does not require the extra projection step and can benefit from the small heterogeneity compared to \textsc{FedPOSGD}. These results also imply that multi-point feedback can be
beneficial in federated adversarial linear bandits.


\section{Conclusion and Future Work}
In this paper, we show that, in the setting of distributed bandit convex optimization against an adaptive
adversary, the benefit of collaboration is very similar to the stochastic setting with first-order gradient information, where the collaboration is useful when: (i) there is stochasticity in the problem and (ii) the variance of the gradient estimator is \textit{``high"} \citep{woodworth2021min} and reduces with collaboration. There are several open questions and future research directions: 
\begin{enumerate}  
    \item Is the final term tight in Theorems \ref{thm:bd_grad_first_stoch} and \ref{thm:smooth_first_stoch}? We do not know any lower bounds in the intermittent communication setting against an adaptive adversary. 
    \item When $K$ is large, but $R$ is a fixed constant, the average regret of the non-collaborative baseline goes to zero, but our upper bounds for \textsc{FedOSGD} do not. It is unclear if our analysis is loose or if we need another algorithm.
    \item What structures in the problem can we further exploit to reduce communication, especially in the bandit feedback setting (c.f., federated stochastic linear bandits \cite{huang2021federated})? 
\end{enumerate}

\section*{Acknowledgments}
We thank the anonymous reviewers at ICML 2023 and OPT-2022, who helped us improve the writing of the paper. We would
also like to thank Ayush Sekhari, John Duchi, and Ohad Shamir for their helpful discussions. This research was partly supported by the NSF-BSF award 1718970, the NSF TRIPOD IDEAL
award, and the NSF-Simons funded Collaboration on the Theoretical Foundations of Deep Learning.


\bibliography{bibliography}
\bibliographystyle{plainnat}

\newpage
\appendix
\onecolumn

\section{Relevant Lower Bound Proofs}\label{sec:lower_bounds}

We want to prove a lower bound for the adversary class $\ppp(\fff^{G})$ with two-point bandit feedback. In particular, we want to show that
\begin{align}
\rrr(\ppp(\fff^{G}, \zeta), \aaa^{0,2}_{IC}) \succeq \frac{GB}{\sqrt{KR}} + \frac{GB\sqrt{d}}{\sqrt{MKR}}.
\end{align}
To prove this, we'd use the reduction $(P2)\succeq (P4) \succeq (P5) \succeq (P6)$. Then we note for the problem (P6), $\zeta=0$, and using $2$-point feedback; we get in total $2MKR$ function value accesses to $\ddd$. We can then use the lower bound in Proposition 2 by \citet{duchi2015optimal} for the problem (P6) for $2M$ points of feedback and $KR$ iterations. Combined with the order-independent lower bound which we prove using problem (P5) in theorem \ref{thm:first_lip}, this proves the required result.
\section{Proof of Theorem \ref{thm:favlb}}
In this section, we provide the proofs of Theorem \ref{thm:favlb}. We first introduce several notations, which will be used in our analysis.
Let $d(x,y)=\|x\|_2^2/2-\|\hat y\|_2^2/2-\langle y,x-\hat y\rangle$, where $\|x\|_2\leq B$ and $\hat y$ is the projected point of $y$ in to the $\ell_2$-norm ball with radius $B$. We have the following holds
\begin{align}\label{eq:dis_ineq}
    d(x,y)\geq \frac{1}{2}\|x-\hat y\|_2^2.
\end{align}
This is due to the following: if $\|y\|\leq B$, \eqref{eq:dis_ineq} clearly holds. If $\|y\|>B$, we have 
\begin{align*}
    d(x,y)-\frac{1}{2}\|x-\hat y\|_2^2=\langle x-\hat y,\hat y-y\rangle\geq \langle \hat y-\hat y,\hat y-y\rangle=0,
\end{align*}
where the inequality is due to the fact that $\hat y-y=(1-\|y\|_2/B)\hat y$ lies in the opposite direction of $\hat y$, and $x=\hat y$ will minimize the inner product. Now, we are ready to prove the regret of Algorithm \ref{alg:fed_pogd}. 
\begin{proof}
Define the following notations
\begin{align*}
    \bar x_{t}=\frac{1}{M}\sum_{m=1}^Mx_t^m,~~\bar w_t=\text{Proj}(\bar x_{t}),~~w_t^m=\text{Proj}(x_t^m).
\end{align*}
We have
\begin{align}\label{dis_lya}
    d(x^\star,\bar x_{t+1})&=\frac{1}{2}\|x^\star\|_2^2-\frac{1}{2}\|\bar w_{t+1}\|_2^2-\langle \bar x_{t+1},x^\star-\bar w_{t+1}\rangle\nonumber\\
    &=\frac{1}{2}\|x^\star\|_2^2-\frac{1}{2}\|\bar w_{t+1}\|_2^2-\langle \bar x_{t}-\eta\frac{1}{M}\sum_{m=1}^Mg_t^m,x^\star-\bar w_{t+1}\rangle\nonumber\\
    &=\underbrace{\frac{1}{2}\|x^\star\|_2^2-\frac{1}{2}\|\bar w_{t+1}\|_2^2-\langle \bar x_{t},x^\star-\bar w_{t+1}\rangle}_{I_1}\nonumber\\
    &\qquad\underbrace{-\eta\frac{1}{M}\sum_{m=1}^M\langle g_t^m,\bar w_{t+1}-x^\star\rangle}_{I_2},
\end{align}
where the second equality comes from the updating rule of Algorithm \ref{alg:fed_pogd}. For the term $I_1$, we have
\begin{align*}
    I_1&=\frac{1}{2}\|x^\star\|_2^2-\frac{1}{2}\|\bar w_{t+1}\|_2^2-\langle \bar x_{t},x^\star-\bar w_{t+1}\rangle\\
    &=\frac{1}{2}\|x^\star\|_2^2-\frac{1}{2}\|\bar w_{t}\|_2^2-\langle \bar x_{t},x^\star-\bar w_{t}\rangle-\langle \bar x_{t},\bar w_{t}-\bar w_{t+1}\rangle\\
    &\qquad-\frac{1}{2}\|\bar w_{t+1}\|_2^2+\frac{1}{2}\|\bar w_{t}\|_2^2\\
    &=d(x^\star,\bar x_t)-d(\bar w_{t+1},\bar x_t)\\
    &\leq d(x^\star,\bar x_t)-\frac{1}{2}\|\bar w_{t+1}-\bar w_{t}\|_2^2,
\end{align*}
where the last inequality is due to \eqref{eq:dis_ineq}. For the term $I_2$, we have
\begin{align*}
    I_2&=-\eta\frac{1}{M}\sum_{m=1}^M\langle g_t^m,\bar w_{t+1}-x^\star\rangle\\
    &=\underbrace{-\eta\frac{1}{M}\sum_{m=1}^M\langle g_t^m-\nabla f_t^m(w_t^m),\bar w_{t+1}-x^\star\rangle}_{I_{21}}\underbrace{-\eta\frac{1}{M}\sum_{m=1}^M\langle \nabla f_t^m(w_t^m),\bar w_{t+1}-\bar x^\star\rangle}_{I_{22}}.
\end{align*}
For the term $I_{21}$, we have
\begin{align*}
    \ee_t[I_{21}]&=\eta \ee_t\frac{1}{M}\sum_{m=1}^M\langle \nabla f_t^m(w_t^m)-g_t^m,\bar w_{t+1}-x^\star\rangle\\
    &=\eta \ee_t\frac{1}{M}\sum_{m=1}^M\langle \nabla f_t^m(w_t^m)-g_t^m,\bar w_{t+1}-\bar w_t\rangle\\
    &\leq \eta \ee_t\bigg\|\frac{1}{M}\sum_{m=1}^M(\nabla f_t^m(w_t^m)-g_t^m)\bigg\|_2\cdot\|\bar w_{t+1}-\bar w_t\|_2\\
    &\leq \eta\frac{\sigma}{\sqrt{M}}\ee_t\|\bar w_{t+1}-\bar w_t\|_2.
\end{align*}
Thus we have
\begin{align*}
    \ee[I_{21}]\leq \eta\frac{\sigma}{\sqrt{M}}\ee\|\bar w_{t+1}-\bar w_t\|_2.
\end{align*}
For the term $I_{22}$, we have 
\begin{align*}
    I_{22}&=-\eta\frac{1}{M}\sum_{m=1}^M\langle \nabla f_t^m(w_t^m),w_{t}^m-\bar x^\star\rangle-\eta\frac{1}{M}\sum_{m=1}^M\langle \nabla f_t^m(w_t^m),\bar w_{t}-w_{t}^m\rangle\\
    &\qquad-\eta\frac{1}{M}\sum_{m=1}^M\langle \nabla f_t^m(w_t^m),\bar w_{t+1}-\bar w_{t}\rangle\\
    &\leq -\eta\frac{1}{M}\sum_{m=1}^M(f_t^m(w_t^m)-f_t^m(x^\star)+\eta\frac{1}{M}\sum_{m=1}^M\|\nabla f_t^m(w_t^m)\|_2\cdot\|\bar w_t-w_t^m\|_2\\
    &\qquad+\eta^2
    \bigg\|\frac{1}{M}\sum_{m=1}^M\nabla f_t^m(w_t^m)\bigg\|_2^2+\frac{1}{4}\|\bar w_{t+1}-\bar w_{t}\|_2^2
\end{align*}
Therefore, combining \eqref{dis_lya} and the upper bound of $I_1$ and $I_2$, we have
\begin{align*}
\ee  d(x^\star,\bar x_{t+1})&\leq \ee  d(x^\star,\bar x_{t})-\frac{1}{4}\ee \|\bar w_{t+1}-\bar w_{t}\|_2^2+\eta\frac{\sigma}{\sqrt{M}}\ee\|\bar w_{t+1}-\bar w_t\|_2\\
&\qquad-\eta\frac{1}{M}\sum_{m=1}^M\ee(f_t^m(w_t^m)-f_t^m(x^\star)+\eta\frac{1}{M}\sum_{m=1}^M\ee\|\nabla f_t^m(w_t^m)\|_2\cdot\|\bar w_t-w_t^m\|_2\\
&\qquad+\eta^2
    \ee\bigg\|\frac{1}{M}\sum_{m=1}^M\nabla f_t^m(w_t^m)\bigg\|_2^2.
\end{align*}
Therefore, we have
\begin{align*}
    \eta\frac{1}{M}\sum_{m=1}^M\ee(f_t^m(w_t^m)-f_t^m(x^\star)&\leq \ee  d(x^\star,\bar x_{t})-\ee  d(x^\star,\bar x_{t+1})+\eta^2\frac{\sigma^2}{M}\\
    &\qquad+\eta\frac{1}{M}\sum_{m=1}^M\ee\|\nabla f_t^m(w_t^m)\|_2\cdot\|\bar w_t-w_t^m\|_2\\
&\qquad+\eta^2
    \ee\bigg\|\frac{1}{M}\sum_{m=1}^M\nabla f_t^m(w_t^m)\bigg\|_2^2.
\end{align*}
In addition, we have 
\begin{align*}
    \frac{1}{M}\sum_{m=1}^M\ee\|\bar w_t-w_t^m\|_2\leq\frac{1}{M}\sum_{m=1}^M \ee\|\bar x_t-x_t^m\|_2\leq 2\eta (\sigma\sqrt{K} + \zeta K).,
\end{align*}
where the last inequality is due to the linear function and follows the similar proofs of Lemma 8 in \citet{woodworth2020minibatch}. Thus, we  can obtain (the indicator function comes from the fact that if $K=1$, there would be no consensus error)
    \begin{align*}
        \frac{1}{M}\sum_{m\in[M]}\ee\sb{f_t^m(w_t^m) - f_t^m(x^\star}
        &\leq \frac{1}{\eta}\rb{\ee  d(x^\star,\bar x_{t})-\ee  d(x^\star,\bar x_{t+1})} \\
        &\qquad+ \eta\rb{G^2 + \frac{\sigma^2}{M}} + \mathbbm{1}_{K>1}\cdot 2G(\sigma\sqrt{K} + \zeta K)\eta.
    \end{align*}
Since $\ee  d(x^\star,\bar x_{T})\geq \ee \|x^\star-\bar w_T\|_2^2/2\geq 0$ and $\ee  d(x^\star,\bar x_{0})=\|x^\star\|_2^2/2$, summing the above inequality over $t$, we can get
\begin{align*}
    \frac{1}{M}\sum_{t\in [KR], m\in[M]}\ee\sb{f_t^m(w_t^m)-f_t^m(x^\star)} \preceq\frac{B^2}{\eta} +\eta\rb{G^2 + \frac{\sigma^2}{M} + \mathbbm{1}_{K>1}\cdot G(\sigma\sqrt{K} + \zeta K)}T.
\end{align*}
If we choose $\eta$ such that
$$\eta = \frac{B}{G\sqrt{T}}\cdot\min\bigg\{1,\frac{G\sqrt{M}}{\sigma}, \frac{\sqrt{G}}{\mathbbm{1}_{K>1}\sqrt{\sigma }K^{1/4}}, \frac{\sqrt{G}}{\mathbbm{1}_{K>1}\sqrt{\zeta K}}\bigg\},$$

We can get
\begin{align*}
   \frac{1}{MKR}\sum_{t\in [KR], m\in[M]}\ee\sb{f_t^m(w_t^m)-f_t^m(x^\star)} \preceq  
        \frac{GB}{\sqrt{KR}} + \frac{\sigma B}{\sqrt{MKR}} + \mathbbm{1}_{K>1}\cdot \rb{\frac{\sqrt{G\sigma}B}{K^{1/4}\sqrt{R}} + \frac{\sqrt{G\zeta}B}{\sqrt{R}}}.
\end{align*}
To get the regret, we just need to notice that we have the linear function, and thus we have: the smoothed function $\hat f=f$ and $\ee f_t^m(w_t^m)=\ee f_t^m(w_t^m+\delta u_t^m)$, where the expectation is over $u_t^m$. Furthermore, $\|g_t^m\|_2^2=d^2\big(f_t^m(w_t^m+\delta u_t^m)\big)^2\leq d^2G^2(B+\delta)^2/\delta^2\leq 4d^2G^2$, where the last inequality is due the choice of $\delta=B$. Since $\ee g_t^m=\nabla \hat f_t^m(w_t^m)$ and $\ee\|g_t^m-\nabla \hat f_t^m(w_t^m)\|_2^2\leq \ee\|g_t^m\|_2^2$ Therefore, we can plug in $\sigma^2=4d^2G^2$ to get our regret.
\end{proof}
\section{Proof of Theorem \ref{thm:bd_grad_first_stoch}}\label{app:bdg}
In this section and the next one, we consider access to a first-order stochastic oracle as an intermediate step before considering the zeroth-order oracle. Specifically, each machine has access to a \textbf{stochastic gradient} $g_t^m$ of $f_t^m$ at point $x_t^m$, such that it is unbiased and has bounded variance, i.e., for all $x\in\xxx$, $$\ee[g_t^m(x_t^m)|x_t^m]=\nabla f_t^m(x_t^m) \text{ and } \ee\sb{\norm{g_t^m(x_t^m) - \nabla f_t^m(x_t^m)}^2|x_t^m} \leq \sigma^2.$$ 
In Algorithm \ref{alg:fed_ogd}, we constructed a particular stochastic gradient estimator at $x_t^m$ with $\sigma^2=G^2d$. We can define the corresponding problem class $\ppp_{M, K, R}^{1, \sigma}(\fff^{G, B}, \zeta)$ when the agents can access a stochastic first-order oracle. We have the following lemma about this problem class:
\begin{lemma}\label{lemma:bd_grad_first_stoch}
Consider the problem class $\ppp_{M,K,R}^{1, \sigma}(\fff^{G,B}, \zeta)$. If we choose $\eta = \frac{B}{G\sqrt{T}}\cdot\min\bigg\{1,\frac{G\sqrt{M}}{\sigma}$, $\frac{\sqrt{G}}{\mathbbm{1}_{K>1}\sqrt{\sigma K}}, \frac{1}{\mathbbm{1}_{K>1}\sqrt{K}}\bigg\}$, then the models $\{x_t^m\}_{t,m=1}^{T,M}$ of Algorithm \ref{alg:fed_ogd} satisfy the following guarantee:
    $$\frac{1}{MKR}\sum_{t\in [KR], m\in[M]}\ee\sb{f_t^m(x_t^m)-f_t^m(x^\star)} \preceq  
        \frac{GB}{\sqrt{KR}} + \frac{\sigma B}{\sqrt{MKR}} + \mathbbm{1}_{K>1}\cdot \rb{\frac{\sqrt{\sigma G}B}{\sqrt{R}} + \frac{GB}{\sqrt{R}}},$$
    where $x^\star \in \arg\min_{x\in \rr^d}\sum_{t\in [KR]}f_t(x)$, and the expectation is w.r.t. the stochastic gradients.    
\end{lemma}

\begin{remark}
Note that when $K=1$, the upper bound in Lemma \ref{lemma:bd_grad_first_stoch} reduces to the first two terms, both of which are known to be optimal due to lower bounds in the stochastic setting, i.e., against a stochastic online adversary \cite{nemirovski1994efficient, hazan2016introduction}. We now use this lemma to guarantee bandit two-point feedback oracles for the same function class. We recall that one can obtain a stochastic gradient for a \textit{``smoothed-version"} $\hat{f}$ of a Lipschitz function $f$ at any point $x\in\xxx$, using two function value calls to $f$ around the point $x$ \cite{shamir2017optimal, duchi2015optimal}.     
\end{remark}
With this lemma, we can prove Theorem \ref{thm:bd_grad_first_stoch}.
\begin{proof}[Proof of Theorem \ref{thm:bd_grad_first_stoch}]
    First, we consider smoothed functions $$\hat{f}_t^m(x) := \ee_{u\sim Unif(S_{d-1})}[f_t^m(x + \delta u)],$$ for some $\delta>0$ and $S_{d-1}$ denoting the euclidean unit sphere. Based on the gradient estimator proposed by \citet{shamir2017optimal} (which can be implemented with two-point bandit feedback) and Lemma \ref{lemma:bd_grad_first_stoch}, we can get the following regret guarantee (noting that $\sigma \leq  c_1\sqrt{d}G$ for a numerical constant $c_1$, c.f., \cite{shamir2017optimal}): 
    $$\ee\sb{\frac{1}{MKR}\sum_{t\in[KR], m\in[M]}\hat{f}_t^m(\hat{x}_t^m)} - \frac{1}{MKR}\sum_{t\in[KR], m\in[M]}\hat{f}_t^m(x^\star) \preceq \frac{GB}{\sqrt{KR}} + \frac{GB\sqrt{d}}{\sqrt{MKR}} + \mathbbm{1}_{K>1}\cdot\frac{GBd^{1/4}}{\sqrt{R}},$$
    where the expectation is w.r.t. the stochasticity in the stochastic gradient estimator. To transform this into a regret guarantee for $f$ we need to account for two things: 
    \begin{enumerate}
        \item The difference between the smoothed function $\hat{f}$ and the original function $f$. This is easy to handle because both these functions are pointwise close, i.e., $\sup_{x \in \xxx}|f(x)-\hat{f}(x)|\leq G\delta$.
        \item The difference between the points $\hat{x}_t^m$ at which the stochastic gradient is computed for $\hat{f}_t^m$ and the actual points $x_t^{m,1}$ and $x_t^{m,2}$ on which we incur regret while making zeroth-order queries to $f_t^m$. This is also easy to handle because due to the definition of the estimator, $x_t^{m,1}, x_t^{m,1}\in B_\delta(\hat{x}_t^m)$, where $B_\delta(x)$ is the $L_2$ ball of radius $\delta$ around $x$.
    \end{enumerate}
    In light of the last two observations, the average regret between the smoothed and original functions only differs by a factor of $2G\delta$, i.e.,
    \begin{align*}
      &\ee\sb{\frac{1}{2MKR}\sum_{t\in[KR], m\in[M], j\in[2]}f_t^m(x_t^{m,j})} - \frac{1}{MKR}\sum_{t\in[KR], m\in[M]}f_t^m(x^\star)\\
      &\qquad\qquad\preceq G\delta + \frac{GB}{\sqrt{KR}} + \frac{GB\sqrt{d}}{\sqrt{MKR}} + \mathbbm{1}_{K>1}\cdot\frac{GBd^{1/4}}{\sqrt{R}}\\
      &\qquad\qquad\preceq \frac{GB}{\sqrt{KR}} + \frac{GB\sqrt{d}}{\sqrt{MKR}} + \mathbbm{1}_{K>1}\cdot\frac{GBd^{1/4}}{\sqrt{R}},
    \end{align*}
    where the last inequality is due to the choice of $\delta$ such that $\delta \preceq \frac{Bd^{1/4}}{\sqrt{R}}\rb{1+ \frac{d^{1/4}}{\sqrt{MK}}}$.
\end{proof}

\section{Proof of Theorem \ref{thm:smooth_first_stoch}}\label{app:smth}
Similar to before, we start by looking at $\ppp_{M,K,R}^{1, \sigma}(\fff^{G,H, B}, \zeta, F_\star)$. We have the following lemma.
\begin{lemma}\label{lem:smooth_first_stoch}
Consider the problem class $\ppp_{M,K,R}^{1, \sigma}(\fff^{G,H, B}, \zeta, F_\star)$.  The models $\{x_t^{m}\}_{t,m=1}^{T,M}$ of Algorithm \ref{alg:fed_ogd} with appropriate $\eta$ (specified in the proof) satisfy the following regret guarantee:
    \begin{align*}
      \frac{1}{MKR}\sum_{t\in [KR], m\in[M]}\ee\sb{f_t^m(x_t^m)-f_t^m(x^\star)} &\preceq \frac{HB^2}{KR} + \frac{\sigma B}{\sqrt{MKR}} + \min\cb{\frac{GB}{\sqrt{KR}}, \frac{\sqrt{HF_\star}B}{\sqrt{KR}}},\\
        & +\mathbbm{1}_{K>1}\cdot\min\Bigg\{\frac{H^{1/3}B^{4/3}\sigma^{2/3}}{K^{1/3}R^{2/3}} + \frac{H^{1/3}B^{4/3}\zeta^{2/3}}{R^{2/3}} + \frac{\sqrt{\zeta\sigma}B}{K^{1/4}\sqrt{R}} + \frac{\zeta B}{\sqrt{R}},\\ 
        & \frac{\sqrt{G\sigma}B}{K^{1/4}\sqrt{R}} + \frac{\sqrt{G\zeta}B}{\sqrt{R}}\Bigg\},  
    \end{align*}
    where $x^\star \in \arg\min_{x\in \rr^d}\sum_{t\in [KR]}f_t(x)$, and the expectation is w.r.t. the stochastic gradients. The models also satisfy the guarantee of Lemma \ref{lemma:bd_grad_first_stoch} with the same step-size.    
\end{lemma}

\begin{proof}[Proof of Theorem \ref{thm:smooth_first_stoch}]
Given Lemma \ref{lem:smooth_first_stoch}, it is now straightforward to prove Theorem \ref{thm:smooth_first_stoch} similar to the proof for Theorem \ref{thm:bd_grad_first_stoch} by replacing $\sigma^2$ with $G^2d$ ad choosing small enough $\delta$ such that $G\delta \cong$ the r.h.s. of the theorem statement.
\end{proof}


\section{Proof of Lemma \ref{lemma:bd_grad_first_stoch}}
In this section, we prove Lemma \ref{lemma:bd_grad_first_stoch}.
\begin{proof}[Proof of Lemma \ref{lemma:bd_grad_first_stoch}]
    Consider any time step $t\in[KR]$ and define ghost iterate $\bar{x}_t = \frac{1}{M}\sum_{m\in[M]}x_t^m$ (which not might actually get computed). If $K=1$, the machines calculate the stochastic gradient at the same point, $\bar{x}_t$. Then using the update rule of Algorithm \ref{alg:fed_ogd}, we can get the following:
    \begin{align*}
        \ee_t\sb{\norm{\bar{x}_{t+1}-x^\star}^2} &= \ee_t\sb{\norm{\bar{x}_t - \frac{\eta_t}{M}\sum_{m\in[M]}\nabla f_t^m(x_t^m) - x^\star + \frac{\eta_t}{M}\sum_{m=1}^{M}\rb{\nabla f_t^m(x_t^m) - g_t^m(x_t^m)}}^2}\\
        &= \norm{\bar{x}_{t}-x^\star}^2 + \frac{\eta_t^2}{M^2}\norm{\sum_{m\in[M]}\nabla f_t^m(x_t^m)}^2 - \frac{2\eta_t}{M}\sum_{m\in[M]}\inner{\bar{x}_t-x^\star}{\nabla f_t^m(x_t^m)} + \frac{\eta_t^2\sigma^2}{M}\\
        &= \norm{\bar{x}_{t}-x^\star}^2 + \frac{\eta_t^2}{M^2}\norm{\sum_{m\in[M]}\nabla f_t^m(x_t^m)}^2 - \frac{2\eta_t}{M}\sum_{m\in[M]}\inner{x^m_t-x^\star}{\nabla f_t^m(x_t^m)}\\
        &\quad + \mathbbm{1}_{K>1}\cdot\frac{2\eta_t}{M}\sum_{m\in[M]}\inner{x^m_t-\bar{x}_t}{\nabla f_t^m(x_t^m)} + \frac{\eta_t^2\sigma^2}{M}\nonumber\\
        &\leq \norm{\bar{x}_{t}-x^\star}^2 + \frac{\eta_t^2}{M^2}\norm{\sum_{m\in[M]}\nabla f_t^m(x_t^m)}^2 - \frac{2\eta_t}{M}\sum_{m\in[M]}\rb{f_t^m(x_t^m) - f_t^m(x^\star)} \\
        &\quad + \mathbbm{1}_{K>1}\cdot\frac{2\eta_t}{M}\sum_{m\in[M]}\inner{x^m_t-\bar{x}_t}{\nabla f_t^m(x_t^m)} + \frac{\eta_t^2\sigma^2}{M},\nonumber\\
    \end{align*}
    where $\ee_t$ is the expectation conditioned on the filtration at time $t$ under which $x_t^m$'s are measurable, and the last inequality is due to the convexity of each function. Re-arranging this leads to 
    \begin{align}
        \frac{1}{M}\sum_{m\in[M]}\rb{f_t^m(x_t^m) - f_t^m(x^\star)} &\leq \frac{1}{2\eta_t}\rb{\norm{\bar{x}_{t}-x^\star}^2 - \ee_t\sb{\norm{\bar{x}_{t+1}-x^\star}^2}} +  \frac{\eta_t}{2 M^2}\norm{\sum_{m\in[M]}\nabla f_t^m(x_t^m)}^2 \nonumber\\
        &\quad + \mathbbm{1}_{K>1}\cdot\frac{1}{M}\sum_{m\in[M]}\ee_t\inner{x^m_t-\bar{x}_t}{\nabla f_t^m(x_t^m)} + \frac{\eta_t\sigma^2}{2M}\nonumber\\
        &\leq \frac{1}{2\eta_t}\rb{\norm{\bar{x}_{t}-x^\star}^2 - \ee_t\sb{\norm{\bar{x}_{t+1}-x^\star}^2}} + \frac{\eta_t}{2}\rb{G^2 + \frac{\sigma^2}{M}}\nonumber\\
        &\quad + \mathbbm{1}_{K>1}\cdot\frac{G}{M}\sum_{m\in[M]}\ee\sb{\norm{x_t^m-\bar{x}_t}}\label{eq:inter1}.
    \end{align}
    The last inequality comes from each function's $G$-Lipschitzness.
    For the last term in \eqref{eq:inter1}, we can upper bound it similar to Lemma 8 in \citet{woodworth2020minibatch} to get that
    \begin{align}\label{eq:concensus}
        \frac{1}{M}\sum_{m\in[M]}\ee\sb{\norm{x_t^m-\bar{x}_t}} & \leq 2(\sigma+G)K\eta.
    \end{align}
     Plugging \eqref{eq:concensus} into \eqref{eq:inter1} and choosing a constant step-size $\eta$, and taking full expectation we get
    \begin{align*}
        \frac{1}{M}\sum_{m\in[M]}\ee\sb{f_t^m(x_t^m) - f_t^m(x^\star)}
        &\leq \frac{1}{2\eta}\rb{\norm{\ee\sb{\bar{x}_{t}-x^\star}^2} - \ee\sb{\norm{\bar{x}_{t+1}-x^\star}^2}} \\
        &\qquad+ \frac{\eta}{2}\rb{G^2 + \frac{\sigma^2}{M}} + \mathbbm{1}_{K>1}\cdot 2G(\sigma + G)K\eta.
    \end{align*}
    Summing this over time $t\in[KR]$ we get,
    \begin{align*}
        \frac{1}{M}\sum_{m\in[M], t\in[T]}\ee\sb{f_t^m(x_t^m) - f_t^m(x^\star)} &\preceq \frac{\norm{\bar{x}_{0}-x^\star}^2}{\eta} + \eta\rb{G^2 + \frac{\sigma^2}{M} + \mathbbm{1}_{K>1}\cdot\sigma G K + \mathbbm{1}_{K>1}\cdot\zeta G K}T\\
        &\preceq \frac{B^2}{\eta} + \eta\rb{G^2 + \frac{\sigma^2}{M} + \mathbbm{1}_{K>1}\cdot\sigma G K + \mathbbm{1}_{K>1}\cdot G^2K}T.
    \end{align*}
    Finally choosing, $$\eta = \frac{B}{G\sqrt{T}}\cdot\min\bigg\{1,\frac{G\sqrt{M}}{\sigma}, \frac{\sqrt{G}}{\mathbbm{1}_{K>1}\sqrt{\sigma K}}, \frac{1}{\mathbbm{1}_{K>1}\sqrt{K}}\bigg\},$$
    we can obtain,
    \begin{align}
        \frac{1}{M}\sum_{m\in[M], t\in[T]}\ee\sb{f_t^m(x_t^m) - f_t^m(x^\star)} &\preceq GB\sqrt{T} + \mathbbm{1}_{K>1}\cdot\sqrt{\sigma G}B\sqrt{KT} + \mathbbm{1}_{K>1}\cdot GB\sqrt{KT} + \frac{\sigma B\sqrt{T}}{\sqrt{M}}.
    \end{align}
    Dividing by $KR$ finishes the proof.
\end{proof}

\section{Proof of Lemma \ref{lem:smooth_first_stoch}}
In this section, we provide the proofs for Lemma \ref{lem:smooth_first_stoch}.
\begin{proof}[Proof of Lemma \ref{lem:smooth_first_stoch}]
    Consider any time step $t\in[KR]$ and define ghost iterate $\bar{x}_t = \frac{1}{M}\sum_{m\in[M]}x_t^m$ (which not might actually get computed). Then using the update rule of Algorithm \ref{alg:fed_ogd}, we can get:
    \begin{align*}
        \ee_t\sb{\norm{\bar{x}_{t+1}-x^\star}^2} &= \ee_t\sb{\norm{\bar{x}_t - \frac{\eta_t}{M}\sum_{m\in[M]}\nabla f_t^m(x_t^m) - x^\star + \frac{\eta_t}{M}\sum_{m=1}^{M}\rb{\nabla f_t^m(x_t^m) - g_t^m(x_t^m)}}^2},\\
        &= \norm{\bar{x}_{t}-x^\star}^2 + \frac{\eta_t^2}{M^2}\norm{\sum_{m\in[M]}\nabla f_t^m(x_t^m)}^2 - \frac{2\eta_t}{M}\sum_{m\in[M]}\inner{\bar{x}_t-x^\star}{\nabla f_t^m(x_t^m)} + \frac{\eta_t^2\sigma^2}{M}\\
        &= \norm{\bar{x}_{t}-x^\star}^2 + \frac{\eta_t^2}{M^2}\norm{\sum_{m\in[M]}\nabla f_t^m(x_t^m)}^2 - \frac{2\eta_t}{M}\sum_{m\in[M]}\inner{x^m_t-x^\star}{\nabla f_t^m(x_t^m)}\\
        &\quad + \mathbbm{1}_{K>1}\cdot\frac{2\eta_t}{M}\sum_{m\in[M]}\inner{x^m_t-\bar{x}_t}{\nabla f_t^m(x_t^m)} + \frac{\eta_t^2\sigma^2}{M}\nonumber\\
        &\leq \norm{\bar{x}_{t}-x^\star}^2 + \frac{\eta_t^2}{M^2}\norm{\sum_{m\in[M]}\nabla f_t^m(x_t^m)}^2 - \frac{2\eta_t}{M}\sum_{m\in[M]}\rb{f_t^m(x_t^m) - f_t^m(x^\star)} \\
        &\quad + \mathbbm{1}_{K>1}\cdot\frac{2\eta_t}{M}\sum_{m\in[M]}\inner{x^m_t-\bar{x}_t}{\nabla f_t^m(x_t^m)} + \frac{\eta_t^2\sigma^2}{M}\nonumber,
    \end{align*}
    where $\ee_t$ is the expectation taken with respect to the filtration at time $t$, and the last line comes from the convexity of each function. Re-arranging this and taking expectation gives 
    leads to
    \begin{align}
        \frac{1}{M}\sum_{m\in[M]}\ee\rb{f_t^m(x_t^m) - f_t^m(x^\star)} &\leq \frac{1}{2\eta_t}\rb{\ee\norm{\bar{x}_{t}-x^\star}^2 - \ee\sb{\norm{\bar{x}_{t+1}-x^\star}^2}} +  \textcolor{blue}{\frac{\eta_t}{2 M^2}\ee\norm{\sum_{m\in[M]}\nabla f_t^m(x_t^m)}^2}\nonumber\\
        &\quad + \mathbbm{1}_{K>1}\cdot\textcolor{red}{\frac{1}{M}\sum_{m\in[M]}\ee\inner{x^m_t-\bar{x}_t}{\nabla f_t^m(x_t^m)}} + \frac{\eta_t\sigma^2}{2M}\label{eq:smth_gen_recursion}
    \end{align}
    \paragraph{Bounding the blue term.} We consider two different ways to bound the term. First note that similar to Lemma \ref{lemma:bd_grad_first_stoch} we can just use the following bound,
    \begin{align}
        \textcolor{blue}{\frac{\eta_t}{2 M^2}\ee\norm{\sum_{m\in[M]}\nabla f_t^m(x_t^m)}^2} &\leq \frac{\eta_t G^2}{2} \label{eq:blue_bd_one}
    \end{align}
    However, since we also have smoothness, we can use the self-bounding property (c.f., Lemma 4.1 \cite{srebro2010optimistic}) to get,
    \begin{align}
        \textcolor{blue}{\frac{\eta_t}{2 M^2}\ee\norm{\sum_{m\in[M]}\nabla f_t^m(x_t^m)}^2} &\leq \frac{\eta_t H}{2M}\sum_{m\in[M]}\rb{f_t^m(x_t^m) - f_t^m(x_t^\star)} + \frac{\eta_t H}{2M} \sum_{m\in[M]}f_t^m(x_t^\star),\nonumber\\
        &\leq \frac{\eta_t H}{2M} \sum_{m\in[M]}f_t^m(x^\star),\label{eq:blue_bd_two}
    \end{align}
    where $x_t^\star$ is the optimizer of $\frac{1}{M}\sum_{m\in[M]}f_t^m(x)$.

    \paragraph{Bounding the red term.} We will bound the term in three different ways. Similar to lemma \ref{lemma:bd_grad_first_stoch}, we can bound the term after taking expectation and then bounding the consensus term similar to Lemma 8 in \citet{woodworth2020minibatch} as follows,
    \begin{align}\label{eq:red_bd_one}
        \textcolor{red}{\frac{1}{M}\sum_{m\in[M]}\ee\sb{\inner{x^m_t-\bar{x}_t}{\nabla f_t^m(x_t^m)}}} &\leq \frac{G}{M}\sum_{m\in[M]}\ee\sb{\norm{x_t^m-\bar{x}_t}}\nonumber\\
        &\leq 2G(\sigma + G)\sum_{t'=\tau(t)}^{\tau(t)+K - 1}\eta_{t'},
    \end{align}
    where $\tau(t)$ maps $t$ to the last time step when communication happens. Alternatively, we can use smoothness as follows after assuming $\eta_t\leq 1/2H$,
    \begin{align}\label{eq:red_bd_two}
        \textcolor{red}{\frac{1}{M}\sum_{m\in[M]}\ee\sb{\inner{x^m_t-\bar{x}_t}{\nabla f_t^m(x_t^m)}}} &= \frac{1}{M}\sum_{m\in[M]}\ee\sb{\inner{x^m_t-\bar{x}_t}{\nabla f_t^m(x_t^m) - \nabla f_t(\bar{x}_t)}},\nonumber\\
        &\leq \sqrt{\frac{1}{M}\sum_{m\in[M]}\ee\norm{x^m_t-\bar{x}_t}^2}\sqrt{\frac{1}{M}\sum_{m\in[M]}\ee\norm{\nabla f_t^m(x_t^m) - \nabla f_t(\bar{x}_t)}^2},\nonumber\\
        &\leq \sqrt{\frac{1}{M}\sum_{m\in[M]}\ee\norm{x^m_t-\bar{x}_t}^2}\sqrt{\frac{2}{M}\sum_{m\in[M]}H^2\ee\norm{x^m_t-\bar{x}_t}^2 + 2\zeta^2},\nonumber\\
        &\leq \frac{2H}{M}\sum_{m\in[M]}\ee\norm{x^m_t-\bar{x}_t}^2 + 2\zeta\sqrt{\frac{1}{M}\sum_{m\in[M]}\ee\norm{x^m_t-\bar{x}_t}^2},\nonumber\\
        &\leq 2\eta_t^2H(\sigma^2 K + \zeta^2K^2) + 2\eta_t\zeta(\sigma\sqrt{K} + \zeta K),
    \end{align}
    where we used lemma 8 from \citet{woodworth2020minibatch} in the last inequality. We can also use the lipschitzness and smoothness assumption together with a constant step size $\eta<1/2H$ to obtain,
    \begin{align}\label{eq:red_bd_three}
        \textcolor{red}{\frac{1}{M}\sum_{m\in[M]}\ee\sb{\inner{x^m_t-\bar{x}_t}{\nabla f_t^m(x_t^m)}}} &\leq \frac{G}{M}\sum_{m\in[M]}\ee\sb{\norm{x_t^m-\bar{x}_t}}\nonumber\\
        &\leq \eta G(\sigma\sqrt{K} + \zeta K).
    \end{align}
    \paragraph{Combining everything.} After using a constant step-size $\eta$, summing \eqref{eq:smth_gen_recursion} over time, we can use the upper bound of the red and blue terms in different ways. If we plug in \eqref{eq:blue_bd_one} and \eqref{eq:red_bd_one} we recover the guarantee of lemma \ref{lemma:bd_grad_first_stoch}. This is not surprising because $\fff^{G,H,B}\subseteq \fff^{G,B}$. Combining the upper bounds in all other combinations assuming $\eta<\frac{1}{2H}$, we can show the following upper bound
    \begin{align*}
        \frac{Reg(M,K,R)}{KR} &\preceq \frac{HB^2}{KR} + \frac{\sigma B}{\sqrt{MKR}} + \min\cb{\frac{GB}{\sqrt{KR}}, \frac{\sqrt{HF_\star}B}{\sqrt{KR}}},\\
        &\quad +\mathbbm{1}_{K>1}\min\cb{\frac{H^{1/3}B^{4/3}\sigma^{2/3}}{K^{1/3}R^{2/3}} + \frac{H^{1/3}B^{4/3}\zeta^{2/3}}{R^{2/3}} + \frac{\sqrt{\zeta\sigma}B}{K^{1/4}\sqrt{R}} + \frac{\zeta B}{\sqrt{R}}, \frac{\sqrt{G\sigma}B}{K^{1/4}\sqrt{R}} + \frac{\sqrt{G\zeta}B}{\sqrt{R}}},
    \end{align*}
    where we used step size,
    \begin{align*}
        \eta &= \min\Bigg\{\frac{1}{2H}, \frac{B\sqrt{M}}{\sigma \sqrt{KR}}, \max\cb{\frac{B}{G\sqrt{KR}}, \frac{B}{\sqrt{HF_\star KR}}},\\ 
        &\qquad\frac{1}{\mathbbm{1}_{K>1}}\cdot\max\Bigg\{\min\cb{\frac{B^{2/3}}{H^{1/3}\sigma^{2/3}K^{2/3}R^{1/3}}, \frac{B^{2/3}}{H^{1/3}\zeta^{2/3}KR^{1/3}}, \frac{B}{K^{3/4}\sqrt{\zeta \sigma R}}, \frac{B}{\zeta K\sqrt{R}}},\\ 
        &\qquad\min\cb{\frac{B}{K^{3/4}\sqrt{G \sigma R}}, \frac{B}{K\sqrt{\zeta GR}}}\Bigg\}\Bigg\}
    \end{align*}
    This finishes the proof.
\end{proof}

\subsection{Modifying the Proof for Federated Adversarial Linear Bandits}\label{sec:modify}
To prove the guarantee for the adversarial linear bandits, we first note that the self-bounding property can't be used anymore as the functions are not non-negative. Thus we proceed with the lemma's proof with the following changes:
\begin{itemize}
    \item We don't prove the additional upper bound in \eqref{eq:blue_bd_two} for blue term.
    \item While upper bounding the red term in \eqref{eq:red_bd_two}, we set $H=0$ and use this single bound for the red term.
\end{itemize}
After making these changes, combining all the terms, and tuning the learning rate, we recover the correct lemma for federated adversarial linear bandits.

\end{document}